\documentclass{article}

\PassOptionsToPackage{numbers, sort&compress}{natbib}
\usepackage[preprint]{neurips_2022}

\usepackage[utf8]{inputenc} % allow utf-8 input
\usepackage[T1]{fontenc}    % use 8-bit T1 fonts
\usepackage[hidelinks]{hyperref}       % hyperlinks
\usepackage{url}            % simple URL typesetting
\usepackage{booktabs}       % professional-quality tables
\usepackage{amsfonts}       % blackboard math symbols
\usepackage{amsmath}       % blackboard math symbols
\usepackage{mathtools}
\usepackage{nicefrac}       % compact symbols for 1/2, etc.
\usepackage{microtype}      % microtypography
\usepackage{xcolor}         % colors
\usepackage[pdftex]{graphicx}

\usepackage[nohyperlinks]{acronym}
\acrodef{CNN}[CNN]{convolutional neural network}
\acrodef{ReLU}[ReLU]{rectified linear unit}
\acrodef{TP}[TP]{target propagation}
\acrodef{AI}[AI]{artificial intelligence}
\acrodef{hEP}[hEP]{holomorphic EP}
\acrodef{MLP}[MLP]{multi-layer perceptron}

\usepackage{amsthm}
\newtheorem{theorem}{Theorem}
\newtheorem{lemma}{Lemma}

\title{Holomorphic Equilibrium Propagation Computes Exact Gradients Through Finite Size Oscillations}

\author{%
  Axel Laborieux$^{1}$ \And Friedemann Zenke$^{1,2}$ \AND
  \texttt{\{firstname.lastname\}@fmi.ch} \\
  $^{1}$ Friedrich Miescher Institute for Biomedical Research, Basel, Switzerland \\
  $^{2}$ Faculty of Natural Sciences, University of Basel, Basel, Switzerland
}

\begin{document}

\maketitle

\begin{abstract}
Equilibrium propagation (EP) is an alternative to backpropagation (BP) that allows the training of deep neural networks with local learning rules. 
It thus provides a compelling framework for training neuromorphic systems and understanding learning in neurobiology. 
However, EP requires infinitesimal teaching signals, thereby limiting its applicability in noisy physical systems. 
Moreover, the algorithm requires separate temporal phases and has not been applied to large-scale problems. Here we address these issues by extending EP to holomorphic networks. 
We show analytically that this extension naturally leads to exact gradients even for finite-amplitude teaching signals. 
Importantly, the gradient can be computed as the first Fourier coefficient from finite neuronal activity oscillations in continuous time without requiring separate phases. 
Further, we demonstrate in numerical simulations that our approach permits robust estimation of gradients in the presence of noise and that deeper models benefit from the finite teaching signals. 
Finally, we establish the first benchmark for EP on the ImageNet $32 \times 32$ dataset and show that it matches the performance of an equivalent network trained with BP. 
Our work provides analytical insights that enable scaling EP to large-scale problems and establishes a formal framework for how oscillations could support learning in biological and neuromorphic systems.
\end{abstract}

\section{Introduction}
\label{intro}
The backpropagation (BP) of error algorithm \cite{rumelhart1986learning} underpins the ability of state-of-the-art deep neural networks to learn useful representations from structured data such as speech, vision, and text \cite{lecun2015deep}.
BP stands out as the most successful algorithm to solve the credit assignment problem in artificial neural networks \cite{lillicrap2020backpropagation, bartunov2018assessing}, which can be defined by the following question: How should a synaptic connection be modified in order to improve the global performance of the network to perform a task, as measured by some objective function?
This is a difficult question since individual synapse may have a complicated influence on downstream processing.
BP solves credit assignment through the chain rule of differentiation \cite{rumelhart1986learning}.
Although BP is efficiently implemented in software, it is difficult to conceive how BP could plausibly be implemented in biological systems.
The problematic aspects are BP's use of symmetric connections and the need for two separate phases: A nonlinear forward pass that propagates neuronal activity and a linear backward pass that carries signed gradient signals \cite{lillicrap2020backpropagation}. 
These two types of processing are also inconvenient for training physical neural networks, since both should be handled by the same circuit, and explicitly propagating errors does not harness the device mismatches typical of neuromorphic hardware \cite{narayanan2017toward, ambrogio2018equivalent}.
Despite its implausibility, representations learned with BP match representations of in-vivo data \cite{yaminS2014performance} better than networks trained with purely biologically-motivated learning rules such as STDP \cite{dan2004spike, tavanaei2019deep}.
This discrepancy raises the question as to whether and how neural dynamics could implement gradient-based credit assignment, and whether it could be as effective as BP to learn useful representations \cite{lillicrap2020backpropagation}.

Equilibrium propagation (EP) \cite{scellier2017equilibrium} is an alternative algorithm for performing credit assignment in dynamical systems that converge to a fixed point, such as energy-based models \cite{hopfield1984neurons}.
EP also proceeds in two phases: In the first phase the dynamical system is presented with static input data until the units settle into an equilibrium or fixed point.
We refer to this state as the free equilibrium.
In a second phase, a teaching signal slightly nudges designated output units towards a target value until the dynamics settle into a second equilibrium that is called the nudged equilibrium.
EP estimates loss gradients by comparing the neuronal activity between the two equilibria.
EP is appealing because the resulting learning rule is spatially local when the energy function consists of two-body interactions, as for instance in continuous Hopfield networks \cite{hopfield1984neurons}. 
Furthermore EP provably approximates the true gradient in the limit of vanishing nudging \cite{scellier2017equilibrium}.
More generally, the implicit differentiation carried out by EP makes it suitable for meta learning \cite{zucchet2021contrastive}, where explicitly backpropagating errors through an inner optimization loop becomes prohibitive due to the high memory requirement of storing intermediate time steps for regular automatic differentiation.

Nevertheless, classic EP \cite{scellier2017equilibrium, ernoult2019updates, laborieux2021scaling} has several limitations.
First, EP estimates only approach the actual loss gradient in the limit of a vanishing nudging or teaching signal.
This requirement makes it impractical for noisy neuromorphic systems where noise can confound small amplitude teaching signals and also unrealistic as a model for learning in the brain where feedback strongly modulates processing.   
Moreover, the mechanisms by which biological circuits could satisfy the requirement for separate phases remains elusive. 
Finally, while EP can train deep networks on CIFAR-10 \cite{laborieux2021scaling}, it has remained an open question whether it can be scaled up to larger and more complex tasks \cite{bartunov2018assessing}. 

In this article, we show that by extending EP with holomorphic network dynamics it naturally estimates exact gradients for finite teaching signals. 
Mathematically, the exact gradients are encoded as a Fourier coefficient of adiabatic neural oscillations. 
This finding suggest a natural way of estimating the gradients online through suitable synaptic filtering operations which dispenses with the need for separate phases, in a similar spirit to \citet{baldi1991contrastive} for contrastive Hebbian learning \cite{movellan1991contrastive}.  
Our main contributions are the following: 
\begin{itemize}
    \item We develop the theory of \ac{hEP} and prove that this allows computing exact gradients locally at synapses from finite teaching signal amplitudes of adiabatic oscillations.
    \item We numerically quantify the accuracy of our estimate and show that it outperforms classic EP, especially in the presence of substrate noise and in deep neural networks.
    \item We demonstrate learning with an always-on oscillating teaching signal, thereby alleviating the need for separated phases. 
    \item Finally, we show that \ac{hEP} achieves the same performance as BP in deep \acp{CNN} trained on CIFAR-10/100 \cite{Krizhevsky09learningmultiple}, and ImageNet $32 \times 32$ \cite{chrabaszcz2017downsampled}.
\end{itemize}

\section{Background and previous work}
\label{background}
\paragraph{Equilibrium propagation (EP).} 
EP \citep{scellier2017equilibrium} allows training convergent dynamical systems to optimize a loss function. 
We denote neuronal unit activity by the vector $\mathbf{s}$, and the learnable parameters such as weights and biases by $\boldsymbol{\theta}$ (Fig.~\ref{fig:setup}a).
The system's dynamics are given as the gradient of a scalar energy function $E(\boldsymbol{\theta}, \mathbf{s})$:
\begin{equation}
    \label{eq:dynamics}
    \frac{\mathrm{d}\mathbf{s}}{\mathrm{d}t} = - \frac{\partial E}{\partial \mathbf{s}}(\boldsymbol{\theta}, \mathbf{s}).
\end{equation}
As a consequence, EP can train any energy-based models, e.g.,  Hopfield networks \cite{hopfield1984neurons} to perform classification \cite{scellier2017equilibrium, ernoult2019updates, laborieux2021scaling}. 
In classic EP, training proceeds in two phases.
First, a subset of units are clamped to the input $\mathbf{x}$ and the system goes to a `free' fixed point denoted by $\mathbf{s}^{\ast}_{0}$.
Second, the loss function $\ell(\boldsymbol{\theta}, \mathbf{s}, \mathbf{y})$, with target $\mathbf{y}$, is scaled with a small positive nudging factor $\beta$ and added to the energy function $E$ which yields the total energy $F(\boldsymbol{\theta}, \mathbf{s}, \beta, \mathbf{y}) := E + \beta\ell(\boldsymbol{\theta}, \mathbf{s}, \mathbf{y})$.
This added teaching signal causes the system to reach a second equilibrium 
$\mathbf{s}^{\ast}_{\beta}$, again by minimizing the total energy $F$.
Although we write that $\ell$ takes all units $\mathbf{s}$ as argument, in practice typically only output units which encode the target label and thus serve as inputs for teaching signals are considered (Fig.~\ref{fig:setup}b).
The learning objective of the system is to optimize the loss function $\ell$ at the free fixed point, which is defined by $\mathcal{L}(\boldsymbol{\theta}, \mathbf{x}, \mathbf{y}) := \ell(\boldsymbol{\theta}, \mathbf{s}^{\ast}_{0}, \mathbf{y})$.
\citet{scellier2017equilibrium} showed that:
\begin{equation}
    \label{eq:origEP}
    \lim_{\beta \to 0}\frac{1}{\beta} \left( \frac{\partial F}{\partial \boldsymbol{\theta}}(\boldsymbol{\theta}, \mathbf{s}^{\ast}_{\beta}, \beta, \mathbf{y}) - \frac{\partial F}{\partial \boldsymbol{\theta}}(\boldsymbol{\theta}, \mathbf{s}^{\ast}_{0}, 0, \mathbf{y}) \right) = \frac{\mathrm{d}\mathcal{L}}{\mathrm{d}\boldsymbol{\theta}}.
\end{equation}
This result requires $F$ to be twice continuously differentiable and assumes that one can apply the implicit function theorem to the equilibrium equation $\partial_{\mathbf{s}}F(\boldsymbol{\theta}, \mathbf{s}^{\ast}_{0}, 0)=0$ (the $\mathbf{y}$ argument is hereafter omitted for clarity), so that $\beta \mapsto \mathbf{s}^{\ast}_{\beta}$ is a continuously differentiable map \cite{scellier2017equilibrium}.
In practice, the left-hand side of Eq.~\eqref{eq:origEP} is estimated by finite differences \cite{scellier2017equilibrium, ernoult2019updates, laborieux2021scaling, luczak2022neurons}.
The appeal of EP for biological plausibility \cite{whittington2019theories, lillicrap2020backpropagation} and neuromorphic hardware \cite{baldi1991contrastive, kendall2020training, zoppo2020equilibrium, stergiou2101refining} arises from the fact that~(i) the system only needs to propagate neural activities (Fig.~\ref{fig:setup}a) and~(ii) in layered neural networks with a Hopfield energy \cite{hopfield1984neurons} (Fig.~\ref{fig:setup}b), the left-hand side of Eq.~\eqref{eq:origEP} can be computed  by a Hebbian-like learning rule as the product of pre- and postsynaptic activity.
In summary, EP implicitly propagates error signals through differences of neuronal activity during the two phases, whereas BP propagates error gradients explicitly \cite{whittington2019theories, lillicrap2020backpropagation}.
However, Eq.~\eqref{eq:origEP} only holds in the limit $\beta \to 0$ where activity differences vanish, which can pose a problem in the presence of noise or when activity differences vanish with network depth, which is related to the problem of vanishing gradients.
In the following, we introduce \acf{hEP} which avoids these issues by estimating exact gradients with finite $\beta$, and thus from finite amplitude teaching signals.

\begin{figure}[tb]
  \centering
  \includegraphics[width=0.9\textwidth]{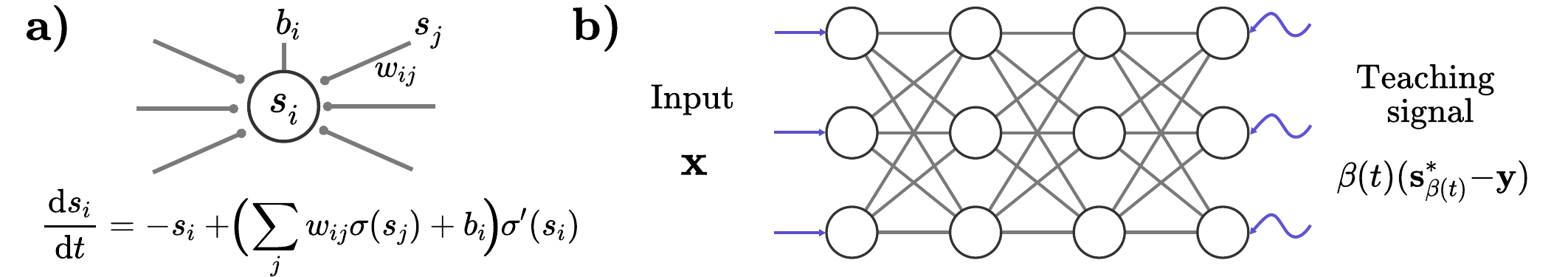}
  \caption{
  \textbf{a)}~Schematic of the neuron model of a continuous Hopfield network \cite{hopfield1984neurons} with holomorphic activation function $\sigma$. 
  The neuron $s_{i}$ receives input both from upstream and downstream neurons $s_{j}$ plus a bias current $b_{i}$. 
  \textbf{b)}~In a typical supervised learning the input $\mathbf{x}$ is clamped, causing the network dynamics to settle into a fixed point. 
  A complex-valued oscillating teaching signal added to the output causes neuronal activity to fluctuate around this fixed point.}
  \label{fig:setup}
\end{figure}

\section{Theoretical results}
\label{sec:theory}

Our main contribution is to show that if $F$ is holomorphic, i.e., differentiable in the sense of complex variables (see Appendix~\ref{sec:app_theory_background} for the definition), \ac{hEP} computes the gradient of the objective function $\mathcal{L}$ for finite $\beta$, i.e., without  requiring the vanishing nudging signals (cf.\ Eq.~\eqref{eq:origEP}).
To accomplish this \ac{hEP} requires a non-vanishing teaching signal that evolves `adiabatically' in the complex plane with respect to the dynamics of the system (Fig.~\ref{fig:setup}b).
In other words, we require the dynamical system to relax to its equilibrium on a much shorter timescale than the timescale of the nudge.

\paragraph{Derivation of holomorphic EP.}
\label{sec:holoEP}
To show that \ac{hEP} yields an unbiased gradient through finite adiabatic nudging, we use the same notation as in Section~\ref{background}.
Specifically, we extend the theory by \citet{scellier2017equilibrium} to the complex case and to dynamical systems, or networks whose  scalar function governing the dynamics is holomorphic.
In line with classic EP we assume that the dynamical system has a free fixed point as described above.

\begin{lemma}[Holomorphic Equilibrium Propagation]
Let~$F$ be a scalar function governing the dynamics, so that the holomorphic implicit function theorem can be applied to the fixed point equation $\partial_{\mathbf{s}}F(\boldsymbol{\theta}, \mathbf{s}^{\ast}_{0}, 0) = 0$, then the gradient formula of equilibrium propagation (Eq.~\eqref{eq:origEP}) holds in the sense of complex differentiation.
\label{lemma:hEP}
\end{lemma}

\begin{proof}
The proof is an extension of the one provided by \citep{scellier2017equilibrium} for the real nudging case. 
The holomorphic implicit function theorem ensures that there exists an open set $U \in \mathbb{C}$ including 0 such that the implicit map $\beta \in U \mapsto \mathbf{s}^{\ast}_{\boldsymbol{\theta}, \beta}$ is holomorphic on $U$.
In particular, the fixed point $\mathbf{s}^{\ast}_{\boldsymbol{\theta}, \beta}$ is defined on $U$
(see Fig.~\ref{fig:overview}b for how this area looks for a toy example).
The proof proceeds in two steps.
First we show that the total derivatives of $F$ with respect to $\boldsymbol{\theta}$ and $\beta$ can still be interchanged for complex variables by virtue of the Schwarz theorem.
Second we show that, at the fixed point, the total derivative of $F$ with respect to $\beta$ ($\boldsymbol{\theta}$) is still equal to the partial derivatives with respect to $\beta$ ($\boldsymbol{\theta}$).
To that end, we apply the chain rule of complex differentiation in:
\begin{equation}
    \label{eq:key_ingredient}
    \frac{\mathrm{d}F}{\mathrm{d}\beta}(\boldsymbol{\theta}, \mathbf{s}_{\boldsymbol{\theta}, \beta}, \beta) = \frac{\partial F}{\partial \beta}(\boldsymbol{\theta}, \mathbf{s}_{\boldsymbol{\theta}, \beta}, \beta) + \frac{\partial F}{\partial \mathbf{s}} \cdot \frac{\partial \mathbf{s}}{\partial \beta}(\boldsymbol{\theta}, \beta) + \frac{\partial F}{\partial \mathbf{\overline{s}}} \cdot \frac{\partial \mathbf{\overline{s}}}{\partial \beta}(\boldsymbol{\theta}, \beta),
\end{equation}
where $\mathbf{\overline{s}}$ denotes the complex conjugate of $\mathbf{s}$.
At equilibrium, the second term on the right hand side cancels by definition of the fixed point, and the third term is zero because $F$ is holomorphic, i.e., its derivative with respect to the conjugate variable is zero according to the Cauchy-Riemann condition~\cite{appel2007mathematics}.
The same argument holds for the derivative with respect to $\boldsymbol{\theta}$.
Therefore, interchanging the total derivatives of $F$ with respect to $\beta$ and $\boldsymbol{\theta}$, and replacing the inner total derivatives by the partial derivatives, we obtain that the EP gradient formula (Eq.~\eqref{eq:origEP}) still holds, but for \emph{complex} differentiation (Appendix~\ref{sec:app_theory_background}):
\begin{equation}
    \label{eq:proof}
    \left. \frac{\mathrm{d}}{\mathrm{d}\beta} \right|_{\beta=0} \left( \frac{\partial F}{\partial \boldsymbol{\theta}} (\boldsymbol{\theta}, \mathbf{s}^{\ast}_{\boldsymbol{\theta}, \beta}, \beta) \right) =  \frac{\mathrm{d}}{\mathrm{d}\boldsymbol{\theta}} \frac{\partial F}{\partial\beta} (\boldsymbol{\theta}, \mathbf{s}^{\ast}_{\boldsymbol{\theta}, \beta}, \beta) = \frac{\mathrm{d}\mathcal{L}}{\mathrm{d}\boldsymbol{\theta}},
\end{equation}
which concludes the proof (a more detailed version is in Appendix \ref{sec:app_proof_lemma}).
\end{proof}
We can now evaluate the left hand side of Eq.~\eqref{eq:proof} using a Cauchy integral (Appendix~\ref{sec:app_theory_background}):
\begin{theorem}[Exact gradient from finite teaching signals]
Assuming that the conditions of Lemma~\ref{lemma:hEP} are met and let $|\beta|>0$ be the radius of a circular path around 0 in $\mathbb{C}$ contained in the open set $U$ on which the fixed point $\mathbf{s}^{\ast}_{\boldsymbol{\theta}, \beta}$ is defined. 
Further assume that this path is parameterized by $t \in [0,T] \mapsto \beta(t) = |\beta|e^{2 \mathrm{i} \pi t/T}$, where $\mathrm{i}$ is the imaginary unit.
Then the loss gradient is given by:
\begin{equation}
    \label{eq:fourier_coeff}
    \frac{\mathrm{d}\mathcal{L}}{\mathrm{d}\boldsymbol{\theta}} = \frac{1}{T|\beta|} \int_{0}^{T} \frac{\partial F}{\partial \boldsymbol{\theta}}\left(\boldsymbol{\theta}, \mathbf{s}^{\ast}_{\boldsymbol{\theta}, \beta(t)}, \beta(t)\right)e^{-2 \mathrm{i} \pi t/T}\mathrm{d}t.
\end{equation}
\label{eq:heqprop_theorem}
\end{theorem}
The full proof is given in Appendix~\ref{sec:app_proof_theorem}.
Theorem~\ref{eq:heqprop_theorem} guarantees that given holomorphic dynamics we can dispense with the requirement of vanishing teaching signal $|\beta| \to 0$ in the limit of `adiabatic' nudging which corresponds to integrating over infinitely many fixed points with a \emph{finite} $|\beta|$.
Note, that complex-valued teaching signals $\beta \in \gamma$ produce fixed points in the complex plane computed through the same equations as in the real case (see Appendix \ref{sec:app_detailed_archi}).
In particular, multi-layered neural networks (Fig.~\ref{fig:setup}b) can be trained by using the continuous Hopfield dynamics \cite{hopfield1984neurons}. 
The trainable parameters are the weights and biases $\boldsymbol{\theta} = (w_{ij}, b_{i})$ and the total energy function $F$ is given by:
\begin{equation}
    \label{eq:tot_ene_hopfield}
    F(\boldsymbol{\theta}, \mathbf{s}, \beta, \mathbf{y}) = \frac{1}{2}\sum_{i}s_{i}^{2} - \frac{1}{2}\sum_{i \neq j}w_{i,j}\sigma(s_{i})\sigma(s_{j}) - \sum_{i}b_{i}\sigma(s_{i}) + \beta \ell(\boldsymbol{\theta}, \mathbf{s}, \mathbf{y}).
\end{equation}
If the activation function $\sigma$ is holomorphic, which is true in the case of sigmoid functions, the same $F$ (Eq.~\eqref{eq:tot_ene_hopfield}) can be evaluated with complex $\beta$, and we can apply Eq.~\eqref{eq:fourier_coeff} to obtain:
\begin{equation}
    \label{eq:hopfield_case}
    - \frac{\mathrm{d}\mathcal{L}}{\mathrm{d}w_{ij}} = \frac{1}{T|\beta|} \int_{0}^{T} \sigma_{i}^{\ast}(t) \sigma_{j}^{\ast}(t) e^{-2 \mathrm{i} \pi t/T}\mathrm{d}t,
\end{equation}
where $\sigma^{\ast}_{i}(t) := \sigma(s^{\ast}_{i, \beta(t)})$.
Therefore, assuming a $T$-periodic teaching signal (Fig.~\ref{fig:setup}b), the gradient is proportional to the first exponential Fourier coefficient of the product of the oscillating activities.
Although this formulation assumes complex neuronal output, we show in Appendix~\ref{sec:app_real_imag} that the gradient in Eq.~\eqref{eq:hopfield_case} can be expressed in terms of the real part or imaginary part only.
The complex teaching signal is therefore best thought of as a way to produce unbiased neuronal oscillations on the nudging timescale.
In the next section, we numerically estimate the Fourier coefficient of Eq.~\eqref{eq:fourier_coeff} with a fixed number of points $N$ on the circle 
which we use to train networks in the subsequent experiments   
and to compare it to the actual loss gradient computed with automatic differentiation.

\paragraph{Numerical estimation of the loss gradient as a Fourier coefficient.}
\label{sec:estimators}

Next, we explain how to estimate the gradient from the corresponding Fourier coefficient (Eq.~\eqref{eq:fourier_coeff}).
In practice, we use a Riemann sum to compute the integral numerically.
We fix the nudging radius $|\beta|>0$ such that the circular path lies in the domain $U$ in which equilibria exist as described in Theorem~\ref{eq:heqprop_theorem}.
We sample the path with $N \geq 2$ nudging points $\{ \beta_{k} := |\beta| e^{2 \mathrm{i} \pi k/N} ~ ; ~ k \in [0, ..., N-1]\}$, and define the estimator:
\begin{equation}
    \label{eq:estimate}
    \hat{\nabla}(N) := \frac{1}{N|\beta|}\sum_{k=0}^{N-1} \frac{\partial F}{\partial \boldsymbol{\theta}}\left(\boldsymbol{\theta}, \mathbf{s}^{\ast}_{\beta_{k}},  \beta_{k}\right)e^{-2 \mathrm{i} \pi k/N}.
\end{equation}
We have that $\hat{\nabla}(N) \underset{N \to \infty}{\longrightarrow} \frac{\mathrm{d}\mathcal{L}}{\mathrm{d}\boldsymbol{\theta}}$, and the remaining bias term when using $N$ points is:
\begin{equation}
    \label{eq:quantitative}
    \hat{\nabla}(N) - \frac{\mathrm{d}\mathcal{L}}{\mathrm{d}\boldsymbol{\theta}} = \sum_{p \equiv 0 ~(N)}^{\infty}\frac{C_{p+1}|\beta|^{p}}{(p+1)!},
\end{equation}
where $C_{p}$ is the $p$-th derivative in $0$ of the function $\beta \mapsto \partial_{\boldsymbol{\theta}}F(\boldsymbol{\theta}, \mathbf{s}^{\ast}_{\beta}, \beta)$ (see Appendix~\ref{sec:app_proof_estimate} for the proof). 
The bias term in Eq.~\ref{eq:quantitative} converges to zero with increasing $N$ because it is a sub-sum of the $(N+1)$-th order remainder of the series expansion of $\beta \mapsto \partial_{\boldsymbol{\theta}}F(\boldsymbol{\theta}, \mathbf{s}^{\ast}_{\beta}, \beta)$ in 0.
The rate of convergence depends on the $C_{p}$ coefficients and the radius $|\beta|$.
In the case $N=2$, the estimate of Eq.~\eqref{eq:estimate} coincides with the `symmetric' estimate of \citet{laborieux2021scaling}.
However, the bias term on the right hand side of Eq.~\eqref{eq:quantitative} is only valid when the dynamics are holomorphic.
Next, we illustrate the approach in a toy experiment, and list three practical improvements brought by \ac{hEP}.

\section{Experiments}
\label{sec:experiments}

In all the experiments, we used the discrete setting of convergent recurrent neural networks of \cite{ernoult2019updates}, and the readout scheme of \cite{laborieux2021scaling} for optimizing the cross-entropy loss function with EP (Appendix~\ref{sec:app_detailed_archi}).
All simulations were implemented in Jax \cite{jax2018github} and Haiku \cite{haiku2020github} (Apache License~2.0).
The datasets were obtained from the Tensorflow datasets library \cite{tensorflow2015-whitepaper}.
Our code is publicly available on GitHub\footnote{\url{https://github.com/Laborieux-Axel/holomorphic_eqprop}}. 
The details of simulations and hyperparameters can be found in Appendix \ref{sec:app_hyperparameters} and \ref{sec:app_simulations}.

\paragraph{Demonstration of holomorphic Equilibrium Propagation on a single data point.}
To provide the first numerical validation of Theorem~\ref{eq:heqprop_theorem} while also allowing us to gain intuitions about dynamics of individual neurons, we implemented a small-scale \ac{MLP} with layer dimensions 6-4-4-4, including input and output layers.
The activation function was a shifted sigmoid, which is holomorphic (Appendix \ref{sec:app_detailed_archi}).
The network was fed with a single datapoint, namely a randomly sampled point from a Gaussian and a random one-hot target.
The dark blue region in Figure~\ref{fig:overview}b shows the map of complex $\beta$ for which the network settles to a fixed point after 200 time steps.
We found experimentally that the area in the complex plane where stable fixed points exist strongly depends on the activation function and the weight initialization (see Appendix \ref{sec:app_fractal}).
As $\beta$ evolves on the circle of radius 0.1 ($N=24$) hidden layer neurons settle into different equilibrium points in the complex plane (Fig.~\ref{fig:overview}a).
We observed that while the teaching signal $\beta(t) = |\beta|e^{2\mathrm{i}\pi t/T}$ was purely sinusoidal, the non-linearity of the network induces neural oscillations that are not purely sinusoidal.
Nevertheless, the gradient is contained in the first mode of these non-linear oscillations (Fig.~\ref{fig:overview}c).

\begin{figure}[tb]
  \centering
  \includegraphics[width=\textwidth]{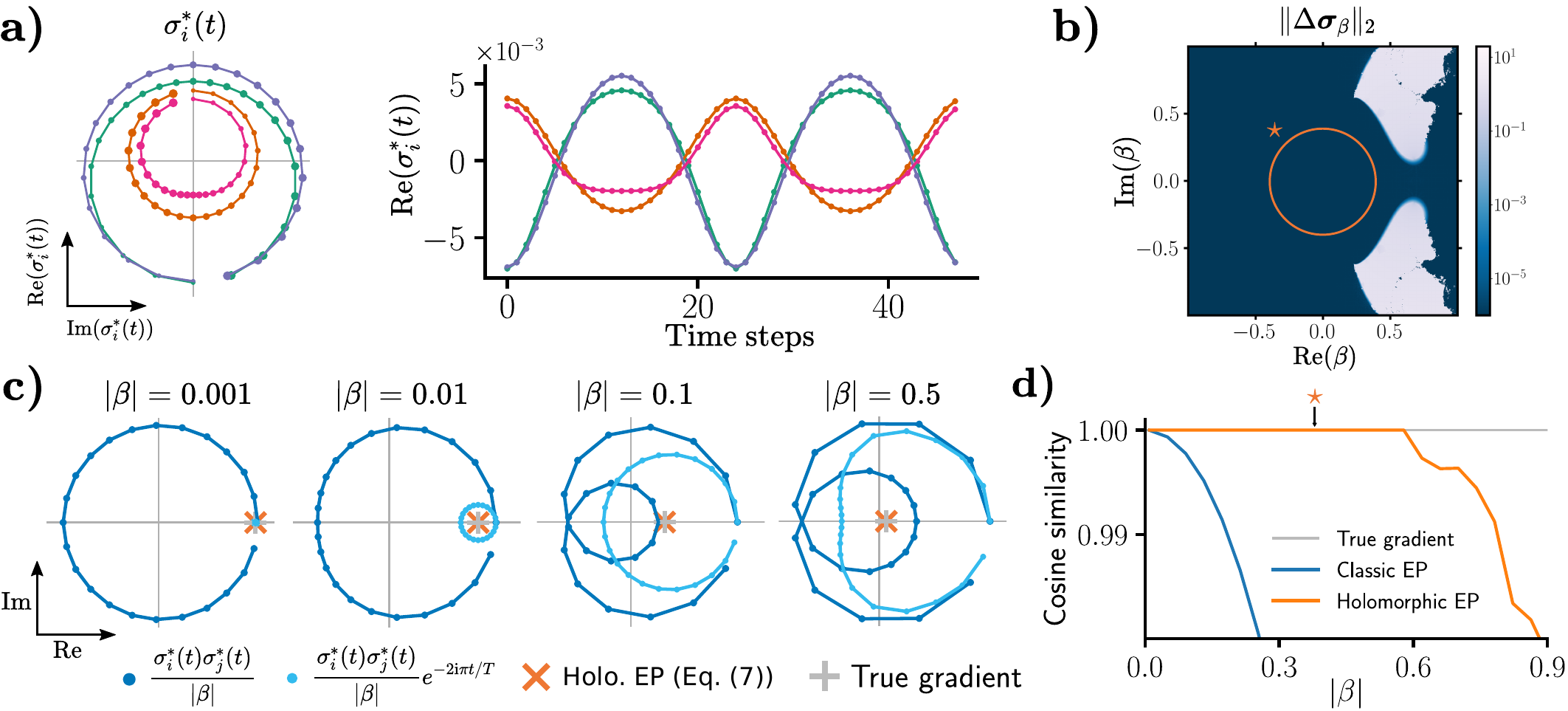}
  \caption{
  Overview of \ac{hEP} for a small \ac{MLP}.
  \textbf{a)}~Neural oscillations sampled at $N=24$ points in the complex plane relative to the free fixed-point in the center (left, point sizes increasing with time).
  The corresponding real part values over two periods (right). 
  \textbf{b)}~Map of the euclidean norm between two consecutive steps of the dynamics in the complex plane spanned by  $\beta$. The map describes the network's dynamical stability.
  Dark blue corresponds to regions with stable fixed points, while
  light blue indicates lack of stability.
  \textbf{c)}~Adiabatic correlations between two neurons for different $|\beta|$ values (dark blue dots).
  Filtering the mode over the period $T$ of the teaching oscillation gives the light blue dots, the temporal mean of which (orange cross) is the gradient (grey plus sign). 
  \textbf{d)}~Cosine similarity with the true gradient of \ac{hEP} (orange), and classic EP \cite{scellier2017equilibrium} (blue).
  The $\star$ marks the teaching radius $|\beta|$ of the path in panel b).
  \Ac{hEP} breaks down when the path passes through unstable (light) regions.
  }
  \label{fig:overview}
\end{figure}

To understand how the gradient is computed accurately when the magnitude of the teaching signal is increased, we recorded the adiabatic product of activities $\sigma_{i}^{\ast}(t) \sigma_{j}^{\ast}(t)$ for one pair of neurons (dark blue) over one teaching period and for increasing values of $|\beta|$ (Fig.~\ref{fig:overview}c).
In the case $|\beta|=0.001$, the perturbation induced by the sinusoidal teaching signal is also purely sinusoidal, and the gradient magnitude is simply the radius of the circle.
However, when $|\beta|$ is increased to 0.01, the linear approximation of the perturbation becomes less accurate, because higher powers of $\beta$ become significant in the series expansion around the free fixed point.
The gradient could still be well approximated by taking the mean of the two radii corresponding to real positive and negative $\beta$, as done by \citet{laborieux2021scaling}.
Increasing $|\beta|$ further to 0.1 and 0.5 yields an even more deformed perturbation, but the gradient is still correctly contained in the first Fourier coefficient of the perturbation.
\Ac{hEP} breaks down when $\beta$ reaches amplitudes for which the corresponding path intersects with areas in which no stable equilibrium exists (cf.~Fig.~\ref{fig:overview}b, light areas, and Fig.~\ref{fig:overview}d).
However, as we will see in the next sections, the finite teaching amplitudes are beneficial when the neuronal dynamics are subject to noise and when training deep neural networks.

\begin{figure}[tbh]
  \centering
  \includegraphics[width=\textwidth]{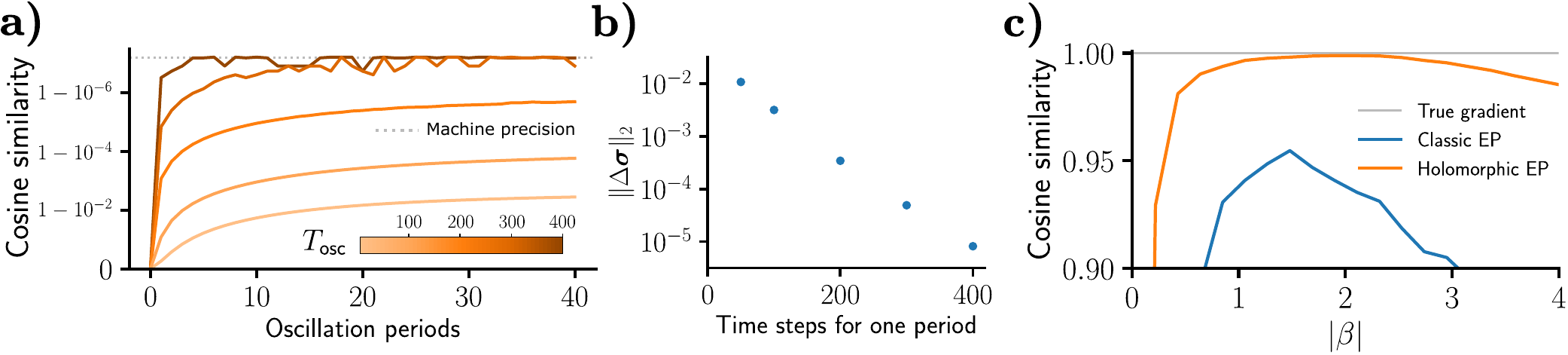}
  \caption{
  \textbf{a)}~Cosine similarity between the true gradient obtained with BP through time and the online estimate ($N=10$) as a function of oscillation periods. 
  Different curves correspond to different oscillation periods (darker color indicates larger $T_{\text{osc}}$).
  $T_{\text{plas}} \approx 10T_{\text{osc}}$ is enough to accurately estimate the gradient. 
  \textbf{b)}~Measure of the residual convergence of the network as a function of the oscillation period $T_{\text{osc}}$, showing that $T_{\text{dyn}} \approx 400/10=40$ time steps.
  \textbf{c)}~Cosine similarity of \ac{hEP} ($N=15$) and classic EP with the true gradient as a function of $|\beta|$ with neuronal output noise.
  All panels used the same \ac{MLP} with two hidden layers of 256 neurons each, fed with a minibatch of ten MNIST samples. 
  }
  \label{fig:cont_noise}
\end{figure}

\paragraph{Holomorphic EP can estimate the gradient in continuous time.}
Our theoretical findings allow us to revisit in a principled way an idea introduced by \citet{baldi1991contrastive} for a continuous-time-implementation of contrastive Hebbian learning \cite{movellan1991contrastive} and learning rules looking to maximize slowness \citep{halvagal2022combination, wiskott2002slow, lipshutz2020biologically}. 
In this context, the oscillating teaching signal is always-on, and the dynamical network is governed by three mechanisms acting on distinct timescales.
The smallest timescale $T_{\text{dyn}}$ is the typical time needed by the network to reach its fixed point.
The second timescale is the period of one teaching oscillation $T_{\text{osc}}$, and the third timescale is the number of periods after which synaptic plasticity occurs $T_{\text{plas}}$.
The gradient can be estimated online by:
\begin{equation}
    \label{eq:continuous}
    \widetilde{\nabla}(T_{\text{plas}}) := -\frac{1}{T_{\text{plas}}|\beta|} \int_{0}^{T_{\text{plas}}} \sigma_{i}(t) \sigma_{j}(t) e^{-2 \mathrm{i} \pi t/T_{\text{osc}}}\mathrm{d}t,
\end{equation}
which converges to the gradient if $T_{\text{dyn}} \ll T_{\text{osc}} \ll T_{\text{plas}}$ (see Appendix \ref{sec:app_proof_continuous}).
This expression is similar to Eq.~(5.2) in \cite{baldi1991contrastive}. 
However their teaching signal oscillates discretely between~0 and~1, and therefore produces a biased estimate of the gradient.
To test the influence of the oscillation timescale $T_{\text{osc}}$ on the online estimate of Eq.~\eqref{eq:continuous}, we compared the online estimation of the gradient over several periods between several values of $T_{\text{osc}}$.
To this end, we used a \ac{MLP} with two hidden layers with 256 units each, which we fed with a minibatch of MNIST data \cite{deng2012mnist}.
We observed that the gradient could be accurately estimated in a few periods for high enough $T_{\text{osc}}$ (Fig.~\ref{fig:cont_noise}a, dark curves).
However, when the oscillations were too fast, a non-vanishing bias remained in the gradient estimates even for many periods (Fig.~\ref{fig:cont_noise}a, lighter curves).
This bias is in all likelihood due to the inability of the system to reach the fixed point (Fig.~\ref{fig:cont_noise}b).
Finally, we found that given appropriate period timings, \ac{hEP} used in the online setting can train a network on MNIST (Table \ref{tab:cont_noise}).
Importantly, the online formulation of \ac{hEP} allows to dispense with the requirement of strictly separate learning phases by replacing them with separate plasticity mechanisms acting on different timescales.

\begin{table}[tbh]
  \caption{MNIST validation errors in \% for classic EP \cite{scellier2017equilibrium}, \ac{hEP}, and online \ac{hEP}, with and without noise. 
  Results are averages ($n=3$) $\pm$ stddev. 
  For training errors see Table~\ref{tab:app_cont_noise} in Appendix~\ref{sec:app_hyperparameters}.}
  \label{tab:cont_noise}
  \centering
  \begin{tabular}{ccccc}
    \toprule
    Noise & Class. EP, $|\beta| = 0.1$ & Class. EP, $|\beta| = 0.4$ & \ac{hEP}, $|\beta| = 0.4$ & Online \ac{hEP} \\
    \midrule
    Noise-free   & 1.87 $\pm$ 0.01 & 2.24 $\pm$ 0.05 & 1.97 $\pm$ 0.08 & 2.05 $\pm$ 0.02 \\
    With noise  & 88.7 $\pm$ 0.0 & 3.01 $\pm$ 0.1 & 1.96 $\pm$ 0.07  & 1.91 $\pm$ 0.16 \\
    \bottomrule
  \end{tabular}
\end{table}

\paragraph{Finite size teaching oscillations provide robustness to noise.}
To analyze \ac{hEP}'s robustness to noise,
we injected a small-amplitude zero-mean Gaussian noise to each neuron in the network in addition to the input from other neurons.
We then used a single minibatch from the MNIST dataset to compute gradient estimates using classic EP and \ac{hEP}.
The latter was computed by using one realisation using $N=15$ points, whereas the classic EP estimate was computed using the free and nudged fixed points  each averaged over $\lceil N/2 \rceil$ to provide a fair comparison.
We found that for small $\beta$ when noise amplitudes were comparable to the activity changes caused by teaching oscillations  the gradient estimate diverged from the true gradient of the noise-free system (Fig.~\ref{fig:cont_noise}c).
To some extent this effect could be mitigated by choosing a finite teaching signal ($\beta\gg0$) \cite{zucchet2021contrastive}. 
However, since $\beta$ also increases the bias for classic EP this creates a trade-off between choosing $\beta$ either too small or too large. 
Importantly, even for the optimal choice of $\beta$, classic EP did not accurately approximate the gradient of the noise-free system.
In contrast, \ac{hEP} thanks to its robustness to finite teaching signals did provide an accurate estimate of the gradient despite the noise.
We verified that \ac{hEP} is indeed more robust to noise than classic EP when used to train the network (Table \ref{tab:cont_noise}). 
Thus, \ac{hEP} combined with finite teaching amplitudes provides an effective way for training noisy computational substrates. 

\begin{figure}
  \centering
  \includegraphics[width=\textwidth]{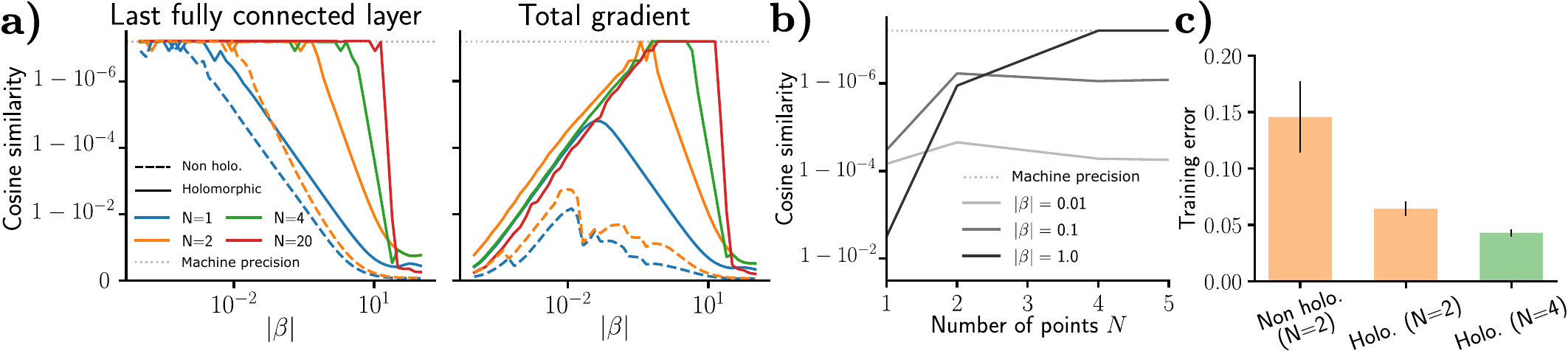}
  \caption{
  \textbf{a)} Cosine similarity between \ac{hEP} and the true gradient in a holomorphic (plain lines) seven-layer VGG-like \ac{CNN} \cite{simonyan2014very}, and a non-holomorphic version using max pooling and \acp{ReLU} (dashed lines).
  Input data is a minibatch of ImageNet $32 \times 32$ \cite{chrabaszcz2017downsampled} consisting of 10 images.
  The left plot shows a comparison of the gradient with respect to the parameters of network's output layer.
  The right plot takes into account the gradient with respect to all network parameters.
  \textbf{b)} Cosine similarity in function of $N$ for three teaching amplitude values. 
  Increasing $N$ is only required for higher amplitudes.
  \textbf{c)} Average training error on CIFAR-10. 
  The average is calculated over three random initializations and error bars correspond to two standard deviations.
  }
  \label{fig:sweep_beta}
\end{figure}

\paragraph{Holomorphic EP matches BP performance on large-scale vision benchmarks.}
To test \ac{hEP}'s ability to train deep neural networks,  
we first investigated  the influence of the number of fixed points $N$ and the teaching amplitude $|\beta|$ on the approximating quality $\hat{\nabla}(N)$ of the loss gradient in a seven-layer VGG-like architecture (\cite{simonyan2014very}; Fig.~\ref{fig:sweep_beta}a,b). 
We ensured holomorphic dynamics by using softmax pooling layers \cite{stergiou2101refining} instead of the non-holomorphic max pooling, and by relying on sigmoid weighted linear units (dSiLU) \cite{elfwing2018sigmoid} (see Appendix \ref{sec:app_detailed_archi}) instead of standard \acp{ReLU}.
We first considered a minibatch of ten images from the ImageNet $32 \times 32$ dataset \cite{chrabaszcz2017downsampled} and computed the gradient using \ac{hEP} as well as BPTT for reference. 
Here, our estimate (Eq.~\eqref{eq:estimate}) was computed by first letting the network settle to the free fixed ($\beta=0$), and then running the phases with complex $\beta$.
We found that the change to holomorphic dynamics already improved upon the gradient estimates used in previous work  \cite{laborieux2021scaling, luczak2022neurons}.
Moreover, we observed that for the last layer increasing $N$ only extended the range of usable teaching magnitudes $|\beta|$, but did not improve the quality of the gradient estimate.
This phenomenon can be understood from Eq.~\eqref{eq:quantitative}, since higher $N$ reduces the bias term considerably, which accommodates higher teaching magnitude $|\beta|$.
However, larger amplitudes tended to improve the total gradient, particularly in deep layers where 
small teaching magnitudes were not enough to produce sufficient error signals 
(see Appendix \ref{sec:app_sweep_beta} for details).
Additionally, larger $N$ were only required when using a higher teaching amplitude (Fig.~\ref{fig:sweep_beta}b).
Finally, we tested how the gradient quality impacts the network training accuracy on CIFAR-10.
We observed that the non-holomorphic VGG was unable to reach low training error (Fig.~\ref{fig:sweep_beta}c), which is consistent with the poor gradient quality (Fig~\ref{fig:sweep_beta}a).
Changing to a holomorphic architecture with the same number of points resulted in a substantial improvement of training accuracy, which was further boosted when training with $N=4$ consistent with our theory.

\begin{table}[tbh]
  \caption{Validation accuracy of BP and \ac{hEP}. 
  All values are averages ($n=3$) $\pm$ stddev.}
  \label{tab:perf}
  \centering
  \begin{tabular}{llllll}
    \toprule
    & CIFAR-10 &  \multicolumn{2}{c}{CIFAR-100} & \multicolumn{2}{c}{ImageNet $32 \times 32$} \\
    \cmidrule(r){2-6}
         & Top-1 (\%)    & Top-1 (\%)  & Top-5 (\%)  & Top-1 (\%) & Top-5 (\%)\\
    \midrule
    BP  &  88.3 $\pm$ 0.1 & 62.0 $\pm$ 0.5  & 86.2 $\pm$ 0.1 & 37.2 $\pm$ 0.4  & 60.9 $\pm$ 0.1 \\
    \ac{hEP}    &  88.6 $\pm$ 0.2  &  61.6 $\pm$ 0.1 & 86.0 $\pm$ 0.1 & 36.5 $\pm$ 0.3 & 60.8 $\pm$ 0.4\\
    \bottomrule
  \end{tabular}
\end{table}

Finally, we wondered whether \ac{hEP} could train deep neural networks on large-scale datasets, which has remained an open problem for most if not all alternative algorithms to BP \cite{bartunov2018assessing}.
To this end, we trained a five-layer \ac{CNN} based on the VGG architecture \cite{simonyan2014very} on multiple vision benchmarks including CIFAR-10, CIFAR-100 \cite{Krizhevsky09learningmultiple}, and the $32 \times 32$ pixel version of ImageNet \cite{chrabaszcz2017downsampled}, which contains 1.2 million data points and 1000 classes like the full ILSVRC dataset \cite{ILSVRC15}.
In all cases we found that the validation accuracy reached by networks trained with BP and \ac{hEP} using only two fixed points ($N=2$) were identical within their uncertainties (Table~\ref{tab:perf}).
Note that the networks trained with BP were the feed-forward equivalent of the holomorphic networks used with EP, but with \acp{ReLU} instead of dSiLU, which did not give satisfactory results.
Thus, \ac{hEP} permits training deep \acp{CNN} on ImageNet $32 \times 32$ to comparable performance levels as standard BP.

\section{Discussion}
\label{sec:discussion}

We have introduced \ac{hEP} which extends classic EP by computing exact loss gradients through integration over finite size adiabatic neuronal oscillations caused by a teaching signal (Section \ref{sec:theory}).
Importantly, such integration can be accomplished online with purely local learning rules which makes it an exciting theoretical framework for studying learning in the brain where oscillations are ubiquitously observed \cite{fell2011role, engel2001dynamic}.
In practice we found that numerically evaluating a small number of points during one oscillation cycle provides an excellent gradient approximation that outperforms classic EP and thanks to the finite oscillation amplitude is robust to noise,
which is an advantage for training neuromorphic hardware systems \cite{kendall2020training, zoppo2020equilibrium, stern2021supervised}.
Additionally, the possibility of using finite teaching signals is conducive for training deep \acp{CNN}, where infinitesimal teaching signals as used by classic EP, may vanish (Section \ref{sec:experiments}).

A body of previous work has attempted to reconcile BP with neurobiology \citep{lillicrap2020backpropagation, payeur2021burst}.
EP is most closely related to classic theories of predictive coding (PC) \cite{rao1999predictive, keller2018predictive, huang2011predictive, whittington2017approximation, millidge2020predictive, rosenbaum2022relationship,friston2010free} which similarly assumes convergent network dynamics cast into an energy minimization problem. 
PC further assumes that errors are encoded in neuronal dynamics, dedicated dendritic compartments, or separate temporal phases \cite{keller2018predictive, whittington2019theories, sacramento2018dendritic, haider2021latent}.
In a similar vein, Target Propagation (TP) \citep{bengio2014auto, lee2015difference, ernoult2022towards} assumes locally encoded error signals which in some cases are obtained by iterating approximate inverses, a property reminiscent of EP \cite{ahmad2020gait, podlaski2020biological, bengio2020deriving} which comes with theoretical guarantees \cite{meulemanS2020theoretical}.
However, all previous EP studies and most of the works above, with two notable exceptions \citep{ernoult2022towards, payeur2021burst}, were all limited to small-scale problems \citep{bartunov2018assessing}.
In contrast, we demonstrate in this article that \ac{hEP} scales to ImageNet $32 \times 32$.

While the ability to run \ac{hEP} online makes it an appealing model for credit assignment in biological neural networks, this interpretation has several notable shortcomings.
First, \ac{hEP} requires complex-valued neuronal outputs and a holomorphic dynamical system which precludes the use of max pooling and \acp{ReLU} and hampers a direct comparison to neurobiology. 
However, we found that holomorphic alternatives exists which empirically yield comparable performance.
Moreover, complex outputs have a long-standing tradition in computational neuroscience where they appear in variations of Hopfield networks \cite{hirose1992dynamics, jankowski1996complex, frady2019robust}, in the framework of theta neurons \cite{ermentrout1986parabolic}, and phasor networks \cite{bybee2022deep} where they are used to describe oscillatory neuronal dynamics.
It is possible to interpret \ac{hEP} within such frameworks.
For instance, it is straightforward to interpret complex neuronal output as oscillating activity with a defined amplitude and relative phase to some reference signal accessible to the entire neuronal population. 
Such a signal could be provided by neuromodulators such as acetylcholine which has been implicated in neural oscillations \citep{hasselmo_role_2006}.
Within our framework, the oscillatory teaching signal then corresponds to a slow phase precession between the neuronal activity and the reference.
Importantly, such a mechanism implies a hierarchy of oscillation frequencies.
Such different oscillations are known to exist in the brain, e.g., theta (4--8\,Hz) and gamma (30--70\,Hz), but their precise purpose remains elusive. 
While establishing formal circuit-level equivalences with \ac{hEP} will require future work, the principled link between oscillatory activity, learning and memory as developed in this article seems promising.
Like other algorithms \ac{hEP} requires symmetric synaptic connectivity between layers which seems biologically implausible.
While theoretical guarantees for exact gradient computation are lost without strict symmetry \cite{scellier2018generalization}, it may not be required for learning \cite{lillicrap2016random, launay2020direct, nokland2016direct}. 
Alternatively, symmetry may be acquired through plastic feedback connections \citep{akrout2019deep, amit2019deep}. 
Although our approach neither uses time-varying input nor neuronal spiking dynamics, spiking extensions to EP have been proposed \cite{o2019training, martin2021eqspike}.
However, applying our theory to time-varying tasks requiring memory will require additional architectural modifications and theoretical concepts \cite{kendall2021gradient} and establishing links to present spike-based approaches \citep{zenke2018superspike, bellec2020solution, payeur2021burst, kaiser2020synaptic}.

Our work augments classic EP with desirable properties for potential neuromorphic applications, which promise power-efficient and equitable \ac{AI} at the edge and in IoT devices \citep{mead1990neuromorphic, indiveri2011neuromorphic, schuman2017survey, grollier2016spintronic}. 
While current software implementation of EP are generally slow compared to backprop, its appeal lies in its potential for training physical networks on future neuromorphic mixed-signal devices that are incompatible with backprop, but achieve settling times on the order of nanoseconds \citep{kendall2020training,stern2021supervised}.
As exciting as such developments are, they also risk negative societal impacts, e.g., through mass surveillance or allowing \ac{AI} systems with potentially discriminatory biases to permeate our everyday lives further. 
A transparent research strategy and taking into account ethical considerations early during product design will be essential to avoid such adverse outcomes. 
On the upside, our theoretical work further consolidates EP as a conceptual framework for understanding the brain, a fundamental requirement to inform future biomedical research targeted at nervous system disorders.

\section*{Acknowledgements}

We thank all members of the Zenke Group for comments and discussions.
This project was supported by the Swiss National Science Foundation [grant number PCEFP3 202981] and the Novartis Research Foundation.

\bibliographystyle{unsrtnat}

\newpage

\appendix

\section{Theoretical proofs}
\label{sec:app_proofs}

In this appendix we detail the proofs of the theoretical results in the body text.

\subsection{Complex analysis background}
\label{sec:app_theory_background}

We recall here the minimal complex analysis background required to appreciate the theoretical results of this work.
In the following, we recall the definitions of holomorphic and Wirtinger derivatives, the Cauchy-Riemann equations and the Cauchy formulas.
We refer the reader to Chapter~4 of [\hyperlink{appel}{S1}] for proofs as well as an excellent introduction to complex analysis.

\paragraph{Definition 1}(Holomorphic function)\textbf{.}
\textit{Let $U$ be an open set of $\mathbb{C}$ and $f: z \in U \mapsto f(z) \in \mathbb{C}$ a function.
$f$ is holomorphic at $a \in U$ if the limit}
\begin{equation*}
    \lim_{z \to a} \frac{f(z)-f(a)}{z-a}
\end{equation*}
\textit{exists.
This limit is then noted $f^{\prime}(a)$.
$f$ is holomorphic on $U$ if it is holomorphic $\forall a \in U$.}

Though this definition looks like the definition of differentiability in $\mathbb{R}$, it brings constraints on the underlying function $\tilde{f}: (x,y) \in \mathbb{R}^{2} \mapsto (\operatorname{Re}(f(x+\mathrm{i}y), \operatorname{Im}(f(x+\mathrm{i}y))$.
The added constraints are the Cauchy-Riemann equations, which can be compactly written after defining Wirtinger derivatives:
\paragraph{Definition 2}(Wirtinger derivatives)\textbf{.}
\textit{Noting $\partial / \partial x$ and $\partial / \partial y$ the usual partial derivatives in $\mathbb{R}^{2}$, the Wirtinger derivatives are defined by:}
\begin{eqnarray*}
    \frac{\partial}{\partial z} := \frac{1}{2}\left(\frac{\partial}{\partial x} - \mathrm{i} \frac{\partial}{\partial y}\right), && \frac{\partial}{\partial \bar{z}} := \frac{1}{2}\left(\frac{\partial}{\partial x} + \mathrm{i} \frac{\partial}{\partial y}\right).
\end{eqnarray*}
In this way, $z$ and its complex conjugate $\bar{z}$ can be thought of as independent variables.
We can then state the Cauchy-Riemann equations as:
\paragraph{Theorem 2}(Cauchy-Riemann equations)\textbf{.}
\textit{if $f$ is holomorphic at $a \in U$, then:}
\begin{eqnarray}
    \label{eq:cauchy_riemann}
    \frac{\partial f}{\partial z}(a) = f^{\prime}(a), && 
    \frac{\partial f}{\partial \bar{z}}(a) = 0.
\end{eqnarray}
These constraints ensure that $f$ is locally expandable everywhere in $U$ into a converging power series.
In particular, it is differentiable at any order and the derivatives can be computed with the:
\paragraph{Theorem 3}(Cauchy formulas)\textbf{.}
\textit{Let $f$ be holomorphic on $U$, let $\gamma$ be any piece-wise continuously differentiable closed curve in $U$ going around $a \in U$ once and counterclockwise, then:}
\begin{equation}
    \label{eq:cauchy_formulas}
    f^{(n)}(a) = \frac{n!}{2 \mathrm{i} \pi} \oint_{\gamma} \frac{f(z)}{(z-a)^{n+1}}\mathrm{d}z.
\end{equation}

\subsection{Proof of Lemma~\ref{lemma:hEP}}
\label{sec:app_proof_lemma}

Here, we give a more detailed proof of holomorphic EP.
We recall the:

\paragraph{Lemma 1}(Holomorphic Equilibrium Propagation)\textbf{.}
\textit{Let~$F$ be a scalar function governing the dynamics, so that the holomorphic implicit function theorem can be applied to the fixed point equation $\partial_{\mathbf{s}}F(\boldsymbol{\theta}, \mathbf{s}^{\ast}_{0}, 0) = 0$, then the gradient formula of equilibrium propagation (Eq.~\eqref{eq:origEP}) holds in the sense of complex differentiation.}

\begin{proof}
We first detail precisely the set of equations on which the holomorphic implicit function theorem is applied.
At the free fixed point $(\boldsymbol{\theta}=\boldsymbol{\theta}_{0}, \beta=0)$ that which exists by assumption, we have the following set of equations:
\begin{equation*}
    \frac{\partial F}{\partial s_{j}}(\boldsymbol{\theta}_{0}, \mathbf{s}_{0}^{\ast}, 0) = 0, \quad 1 \leq j \leq n,
\end{equation*}
where $n$ is the number of units in the system.
The functions $\partial_{s_{j}}F$ are holomorphic by assumption.
If we further assume that the Hessian of $F$ with respect to $\mathbf{s}$ is invertible in $(\boldsymbol{\theta}_{0}, \mathbf{s}^{\ast}_{0}, 0)$, i.e.:
\begin{equation*}
    \det \left( \frac{\partial^{2} F}{\partial s_{i} \partial s_{j}} (\boldsymbol{\theta}_{0}, \mathbf{s}_{0}^{\ast}, 0) \right)_{i,j} \neq 0,
\end{equation*}
then the holomorphic version of the implicit function theorem [\hyperlink{cartan}{S2}] can be applied and there exists an open neighbourhood of $(\boldsymbol{\theta}=\boldsymbol{\theta}_{0}, \beta=0)$ in the complex domain where the implicit map $(\boldsymbol{\theta}, \beta) \mapsto \mathbf{s}^{\ast}_{\boldsymbol{\theta}, \beta}$ is holomorphic, and where the fixed point equations hold:
\begin{equation*}
    \frac{\partial F}{\partial \mathbf{s}}(\boldsymbol{\theta}, \mathbf{s}_{\boldsymbol{\theta}, \beta}^{\ast}, \beta) = 0.
\end{equation*}
At such fixed points, we have that the total derivatives of $F$ with respect to either $\beta$ or $\boldsymbol{\theta}$ are equal to the partial derivatives, which can be seen by applying the chain rule of complex differentiation using Wirtinger derivatives.
There are now in principle three contributions to the total derivative of $F$ with respect to $\beta$:
\begin{equation}
    \label{eq:key_ing_proof}
    \frac{\mathrm{d}F}{\mathrm{d}\beta}(\boldsymbol{\theta}, \mathbf{s}_{\boldsymbol{\theta}, \beta}, \beta) = \frac{\partial F}{\partial \beta}(\boldsymbol{\theta}, \mathbf{s}_{\boldsymbol{\theta}, \beta}, \beta) + \underbrace{\frac{\partial F}{\partial \mathbf{s}}}_{\mathrlap{=0~\text{at a fixed point}}} \cdot \frac{\partial \mathbf{s}}{\partial \beta}(\boldsymbol{\theta}, \beta) + \underbrace{\frac{\partial F}{\partial \mathbf{\overline{s}}}}_{\mathrlap{=0~\text{by Cauchy-Riemann (Eq.~\eqref{eq:cauchy_riemann})}}} \cdot \frac{\partial \mathbf{\overline{s}}}{\partial \beta}(\boldsymbol{\theta}, \beta),
\end{equation}
where $\mathbf{\overline{s}}$ denotes the complex conjugate of $\mathbf{s}$.
At the fixed point however, the second term on the right hand side cancels by definition.
The third term is zero because $F$ is holomorphic, i.e., its derivative with respect to the conjugate variable $\mathbf{\overline{s}}$ is zero according to the Cauchy-Riemann condition~[\hyperlink{appel}{S1}].
The same argument holds for the total derivative with respect to $\boldsymbol{\theta}$.

Finally, the cross-derivatives of $F$ with respect to complex $\beta$ and $\boldsymbol{\theta}$ can be exchanged, which is a consequence of the Schwarz theorem applied to the function $(\boldsymbol{\theta}, \beta) \mapsto G(\boldsymbol{\theta}, \beta) := F(\boldsymbol{\theta}, \mathbf{s}^{\ast}_{\boldsymbol{\theta}, \beta}, \beta)$.
Therefore we have that:
\begin{align*}
    \frac{\partial^{2} G}{ \partial \beta \partial \boldsymbol{\theta}}(\boldsymbol{\theta}, \beta) &= \frac{\partial^{2} G}{\partial \boldsymbol{\theta} \partial \beta}(\boldsymbol{\theta}, \beta), \\
    \frac{\mathrm{d}}{\mathrm{d}\beta} \frac{\mathrm{d}}{\mathrm{d}\boldsymbol{\theta}} F(\boldsymbol{\theta}, \mathbf{s}^{\ast}_{\boldsymbol{\theta}, \beta}, \beta) &= \frac{\mathrm{d}}{\mathrm{d}\boldsymbol{\theta}} \frac{\mathrm{d}}{\mathrm{d}\beta} F(\boldsymbol{\theta}, \mathbf{s}^{\ast}_{\boldsymbol{\theta}, \beta}, \beta), \\
    \frac{\mathrm{d}}{\mathrm{d}\beta} \frac{\partial}{\partial \boldsymbol{\theta}} F(\boldsymbol{\theta}, \mathbf{s}^{\ast}_{\boldsymbol{\theta}, \beta}, \beta) &= \frac{\mathrm{d}}{\mathrm{d}\boldsymbol{\theta}} \frac{\partial}{\partial \beta} F(\boldsymbol{\theta}, \mathbf{s}^{\ast}_{\boldsymbol{\theta}, \beta}, \beta), \quad \text{by Eq.~\eqref{eq:key_ing_proof}.}
\end{align*}
By then applying this equality in $\beta=0$ and $\boldsymbol{\theta} = \boldsymbol{\theta}_{0}$, we obtain the EP gradient formula (Eq.~\eqref{eq:origEP}) for complex differentiation:
\begin{equation*}
    \left. \frac{\mathrm{d}}{\mathrm{d}\beta} \right|_{\beta=0} \left( \frac{\partial F}{\partial \boldsymbol{\theta}} (\boldsymbol{\theta}, \mathbf{s}^{\ast}_{\boldsymbol{\theta}, \beta}, \beta) \right) =  \frac{\mathrm{d}}{\mathrm{d}\boldsymbol{\theta}} \frac{\partial F}{\partial\beta} (\boldsymbol{\theta}, \mathbf{s}^{\ast}_{\boldsymbol{\theta}, \beta}, \beta) = \frac{\mathrm{d}\mathcal{L}}{\mathrm{d}\boldsymbol{\theta}},
\end{equation*}
which concludes the proof.
\end{proof}

\subsection{Proof of Theorem~\ref{eq:heqprop_theorem}}
\label{sec:app_proof_theorem}

\paragraph{Theorem 1}(Exact gradient from finite teaching signals)\textbf{.}
\textit{Assuming that the conditions of Lemma~\ref{lemma:hEP} are met and let $|\beta|>0$ be the radius of a circular path around 0 in $\mathbb{C}$ contained in the open set $U$ on which the fixed point $\mathbf{s}^{\ast}_{\boldsymbol{\theta}, \beta}$ is defined. 
Further assume that this path is parameterized by $t \in [0,T] \mapsto \beta(t) = |\beta|e^{2 \mathrm{i} \pi t/T}$, where $\mathrm{i}$ is the imaginary unit.
Then the loss gradient is given by:}
\begin{equation}
    \label{eq:fourier_coeff_app}
    \frac{\mathrm{d}\mathcal{L}}{\mathrm{d}\boldsymbol{\theta}} = \frac{1}{T|\beta|} \int_{0}^{T} \frac{\partial F}{\partial \boldsymbol{\theta}}\left(\boldsymbol{\theta}, \mathbf{s}^{\ast}_{\boldsymbol{\theta}, \beta(t)}, \beta(t)\right)e^{-2 \mathrm{i} \pi t/T}\mathrm{d}t \quad .
\end{equation}

\begin{proof}
By assumption the fixed point $\beta \mapsto \mathbf{s}^{\ast}_{\boldsymbol{\theta}, \beta}$ is defined on an open set $U$ (by the holomorphic implicit function theorem) containing the disk of radius $|\beta|$ centered around 0. 
In particular, the function $\beta \in U \mapsto \partial_{\boldsymbol{\theta}}F(\boldsymbol{\theta}, \mathbf{s}^{\ast}_{\boldsymbol{\theta}, \beta}, \beta)$, is also holomorphic by composition.
The left hand side of Eq.~\eqref{eq:proof} can thus be computed with the Cauchy formulas (Eq.~\eqref{eq:cauchy_formulas} with $f=\partial_{\boldsymbol{\theta}}F$, $n=1$, $a=0$), and $\gamma$  an arbitrary closed path leading around zero once and counterclockwise in $U$:
\begin{equation}
    \label{eq:cauchy}
    \left. \frac{\mathrm{d}}{\mathrm{d}\beta} \right|_{\beta=0} \left( \frac{\partial F}{\partial \boldsymbol{\theta}} (\boldsymbol{\theta}, \mathbf{s}^{\ast}_{\boldsymbol{\theta}, \beta}, \beta) \right) = \frac{1}{2 \mathrm{i} \pi} \oint_{\gamma} \frac{1}{\beta^{2}}\frac{\partial F}{\partial \boldsymbol{\theta}}(\boldsymbol{\theta}, \mathbf{s}^{\ast}_{\boldsymbol{\theta}, \beta}, \beta)\mathrm{d}\beta \quad .
\end{equation}
To obtain Eq.~\eqref{eq:fourier_coeff_app}, we choose $\gamma$ as a circular path in the complex plane with radius $|\beta|>0$ parameterized by time $t \in [0,T] \mapsto \beta(t) = |\beta|e^{2 \mathrm{i} \pi t/T}$, where $T$ is a full period.
After the change of variable $\mathrm{d}\beta = (\nicefrac{2 \mathrm{i} \pi \beta(t) }{T})\mathrm{d}t$ in Eq.~\eqref{eq:cauchy},  and using  Lemma~\ref{lemma:hEP}, the loss gradient is given by:
\begin{align*}
    \frac{\mathrm{d}\mathcal{L}}{\mathrm{d}\boldsymbol{\theta}} &= \frac{1}{2 \mathrm{i} \pi} \oint_{\gamma} \frac{1}{\beta^{2}}\frac{\partial F}{\partial \boldsymbol{\theta}}(\boldsymbol{\theta}, \mathbf{s}^{\ast}_{\boldsymbol{\theta}, \beta}, \beta)\mathrm{d}\beta \\
    &= \frac{1}{2 \mathrm{i} \pi} \int_{0}^{T} \frac{1}{\beta(t)^{2}}\frac{\partial F}{\partial \boldsymbol{\theta}}(\boldsymbol{\theta}, \mathbf{s}^{\ast}_{\boldsymbol{\theta}, \beta(t)}, \beta(t)) \left(\frac{2 \mathrm{i} \pi \beta(t) }{T}\right)\mathrm{d}t \\
    &= \frac{1}{T} \int_{0}^{T} \frac{1}{\beta(t)} \frac{\partial F}{\partial \boldsymbol{\theta}}\left(\boldsymbol{\theta}, \mathbf{s}^{\ast}_{\boldsymbol{\theta}, \beta(t)}, \beta(t)\right)\mathrm{d}t \\
    &= \frac{1}{T|\beta|} \int_{0}^{T} \frac{\partial F}{\partial \boldsymbol{\theta}}\left(\boldsymbol{\theta}, \mathbf{s}^{\ast}_{\boldsymbol{\theta}, \beta(t)}, \beta(t)\right)e^{-2 \mathrm{i} \pi t/T}\mathrm{d}t.
\end{align*}
\end{proof}

\subsection{Roles of real and imaginary parts in the learning rule}
\label{sec:app_real_imag}

Recall that for the continuous Hopfield network case the partial derivative of $F$ with respect to a parameter $w_{ij}$ is the product of pre and post activation (Eq.~\eqref{eq:tot_ene_hopfield}), so that applying Eq.~\eqref{eq:proof} yields:
\begin{equation*}
    \frac{\mathrm{d}\mathcal{L}}{\mathrm{d}w_{ij}} = \left. \frac{\mathrm{d}}{\mathrm{d}\beta} \right|_{\beta=0} \underbrace{\left( \frac{\partial F}{\partial w_{ij}} (\boldsymbol{\theta}, \mathbf{s}^{\ast}_{\beta}, \beta) \right)}_{= -\sigma(s_{i, \beta}^{\ast})\sigma(s_{j, \beta}^{\ast})} =  - \left. \frac{\mathrm{d}\left(\sigma(s_{i, \beta}^{\ast})\sigma(s_{j, \beta}^{\ast})\right)}{\mathrm{d}\beta} \right|_{\beta=0},
\end{equation*}
which can further be expressed as:
\begin{equation}
     \label{eq:decompo}
     \left.  \frac{\mathrm{d}\left(\sigma(s_{i, \beta}^{\ast})\sigma(s_{j, \beta}^{\ast})\right)}{\mathrm{d}\beta} \right|_{\beta=0} =  \left. \left( \sigma(s_{i, \beta}^{\ast}) \frac{\mathrm{d}\sigma(s_{j, \beta}^{\ast})}{\mathrm{d}\beta} \right) \right|_{\beta=0} + \left. \left( \sigma(s_{j, \beta}^{\ast})  \frac{\mathrm{d}\sigma(s_{i, \beta}^{\ast})}{\mathrm{d}\beta} \right) \right|_{\beta=0}.
\end{equation}
Using the same assumptions as Section \ref{sec:theory}, the map $\beta \in U \mapsto s_{i, \beta}^{\ast}$ is holomorphic, and so is the map
$\beta \in U \mapsto \sigma(s_{i, \beta}^{\ast})$ by composition.
We can thus expand it in a power series around zero:
\begin{equation*}
    \sigma(s_{i, \beta}^{\ast}) = \sum_{k=0}^{\infty} \frac{\beta^{k}}{k!} \left. \frac{\mathrm{d}^{k}\sigma(s_{i, \beta}^{\ast})}{\mathrm{d}\beta^{k}}\right|_{\beta=0}.
\end{equation*}
We can then separate the sum into the real and imaginary parts because the series converge absolutely. 
Assuming that $\beta = |\beta|e^{2\mathrm{i} \pi t/T}$, and applying the Euler formula, we obtain:
\begin{eqnarray}
    \operatorname{Re}\left(\sigma(s_{i, \beta}^{\ast})\right) &=& \sum_{k=0}^{\infty} \cos{\left( \frac{2k \pi t}{T}\right)}\frac{|\beta|^{k}}{k!} \left. \frac{\mathrm{d}^{k}\sigma(s_{i, \beta}^{\ast})}{\mathrm{d}\beta^{k}}\right|_{\beta=0}, \nonumber \\
    \operatorname{Im}\left(\sigma(s_{i, \beta}^{\ast})\right) &=& \sum_{k=1}^{\infty} \sin{\left(\frac{2k \pi t}{T}\right)}\frac{|\beta|^{k}}{k!} \left. \frac{\mathrm{d}^{k}\sigma(s_{i, \beta}^{\ast})}{\mathrm{d}\beta^{k}}\right|_{\beta=0}.
\end{eqnarray}
Therefore, the first derivative ($k=1$) with respect to $\beta$ in Eq.~\eqref{eq:decompo} can be obtained by either projecting the real part against the cosine function, or imaginary part against the sine function:
\begin{align*}
    \left. \frac{\mathrm{d}\sigma(s_{i, \beta}^{\ast})}{\mathrm{d}\beta} \right|_{\beta=0} &= \frac{2}{|\beta|T} \int_{0}^{T} \operatorname{Re}\left(\sigma(s_{i, \beta}^{\ast})\right) \cos{\left( \frac{2 \pi t}{T}\right)} \mathrm{d}t, \\
    &= \frac{2}{|\beta|T} \int_{0}^{T} \operatorname{Im}\left(\sigma(s_{i, \beta}^{\ast})\right) \sin{\left( \frac{2 \pi t}{T}\right)} \mathrm{d}t,
\end{align*}
by orthogonality of the family $\left((t \mapsto \cos{\left( \frac{2k \pi t}{T}\right)})_{k \geq 0}, (t \mapsto \sin{\left( \frac{2k \pi t}{T}\right)})_{k \geq 0}\right)$ in $L^{2}[0, T]$.
Note that the higher order derivatives with respect to $\beta$ can be obtained as well by projecting against the corresponding cosine or sine function.
The same holds for index $j$ by symmetry.
As an interesting final note, if we define $\operatorname{Re}_{1}(\sigma(s_{i, \beta}^{\ast}))$ and $\operatorname{Im}_{1}(\sigma(s_{i, \beta}^{\ast}))$, the first order contributions in $\beta$ to the real and imaginary parts of the neural activity, we find that they are the only ones to contribute to the gradient computation. 
We can appreciate that they are related through $\operatorname{Im}_{1}(\sigma(s_{i, \beta}^{\ast})) = - \frac{T}{2 \pi} \frac{\mathrm{d}}{\mathrm{d}t}\operatorname{Re}_{1}(\sigma(s_{i, \beta}^{\ast}))$, where the time derivative is at the scale of the teaching signal. 

\subsection{Derivation of the bias term}
\label{sec:app_proof_estimate}

Recall the definition of $\beta_{k} := |\beta|e^{2\mathrm{i}\pi k/N}$, for $k \in [0, ..., N-1]$, $N \geq 2$, and the gradient estimate (Eq.~\eqref{eq:estimate}):
\begin{equation*}
    \hat{\nabla}(N) := \frac{1}{N|\beta|}\sum_{k=0}^{N-1} \frac{\partial F}{\partial \boldsymbol{\theta}}\left(\boldsymbol{\theta}, \mathbf{s}^{\ast}_{\beta_{k}},  \beta_{k}\right)e^{-2 \mathrm{i} \pi k/N}.
\end{equation*}
For simplicity of notation, we rewrite $\partial_{\boldsymbol{\theta}}F(\beta) := \frac{\partial F}{\partial \boldsymbol{\theta}}(\boldsymbol{\theta}, \mathbf{s}^{\ast}_{\beta},  \beta)$. 
The function $\beta \mapsto \partial_{\boldsymbol{\theta}}F(\beta)$ is holomorphic on an open set $U$ including zero, and so is $\beta \mapsto \mathbf{s}^{\ast}_{\beta}$ by the holomorphic implicit function theorem.
We assume the $\beta_{k}$ are included in $U$, so that we can expand $\partial_{\boldsymbol{\theta}}F(\beta_{k})$ in a power series around zero:
\begin{equation*}
    \partial_{\boldsymbol{\theta}}F(\beta_{k}) = \sum_{p=0}^{\infty}\frac{\beta_{k}^{p}}{p!}\left[\frac{\mathrm{d}^{p}}{\mathrm{d}\beta^{p}}\partial_{\boldsymbol{\theta}}F\right](0),
\end{equation*}
we define $C_{p}:=\left[\frac{\mathrm{d}^{p}}{\mathrm{d}\beta^{p}}\partial_{\boldsymbol{\theta}}F\right](0)$.
The quantity of interest is $C_{1}$, since it is the gradient of the loss (Eq.~\eqref{eq:proof})
\begin{align*}
    \partial_{\boldsymbol{\theta}}F(\beta_{k}) &= C_{0} + \beta_{k}C_{1} + \sum_{p=2}^{\infty}\frac{\beta_{k}^{p}}{p!}C_{p} \\
    \frac{\partial_{\boldsymbol{\theta}}F(\beta_{k})}{\beta_{k}} &= C_{0}\beta_{k}^{-1} + C_{1} + \sum_{p=2}^{\infty}\frac{\beta_{k}^{p-1}}{p!}C_{p} \\
    \frac{1}{N}\sum_{k=0}^{N-1}\frac{\partial_{\boldsymbol{\theta}}F(\beta_{k})}{\beta_{k}} &= C_{1} + C_{0}\frac{1}{N}\sum_{k=0}^{N-1}\beta_{k}^{-1} + \frac{1}{N}\sum_{k=0}^{N-1}\sum_{p=2}^{\infty}\frac{\beta_{k}^{p-1}}{p!}C_{p}.
\end{align*}
The sum symbols on the right can be interchanged thanks to the absolute convergence of the power series.
\begin{align*}
    \frac{1}{N}\sum_{k=0}^{N-1}\frac{\partial_{\boldsymbol{\theta}}F(\beta_{k})}{\beta_{k}} &= C_{1} + C_{0}\frac{1}{N}\sum_{k=0}^{N-1}\beta_{k}^{-1} + \sum_{p=2}^{\infty}\frac{C_{p}}{p!}\frac{1}{N}\sum_{k=0}^{N-1}\beta_{k}^{p-1} \\
    \frac{1}{N}\sum_{k=0}^{N-1}\frac{\partial_{\boldsymbol{\theta}}F(\beta_{k})}{\beta_{k}} &= C_{1} + C_{0}\frac{1}{N|\beta|}\sum_{k=0}^{N-1}e^{-2\mathrm{i}\pi k/N} + \sum_{p=1}^{\infty}\frac{C_{p+1}}{(p+1)!}\frac{|\beta|^{p}}{N}\sum_{k=0}^{N-1}e^{2\mathrm{i}\pi pk/N}.
\end{align*}
It remains to evaluate the geometric sums of the form $\sum_{k=0}^{N-1}e^{2\mathrm{i}\pi pk/N}$ for $p=-1$ and $p\geq1$.
If $N$ divides $p$, i.e $p \equiv 0~(N)$, then we can write $p = Nq$ and we have:
\begin{equation*}
    \sum_{k=0}^{N-1}e^{2\mathrm{i}\pi pk/N} = \sum_{k=0}^{N-1}e^{2\mathrm{i}\pi qNk/N} = \sum_{k=0}^{N-1}e^{2\mathrm{i}\pi qk} = \sum_{k=0}^{N-1}1 = N.
\end{equation*}
If $N$ does not divide $p$, then the geometric sum of ratio $e^{2\mathrm{i}\pi p/N}$ can be computed:
\begin{equation*}
    \sum_{k=0}^{N-1}e^{2\mathrm{i}\pi pk/N} = \frac{1 - (e^{2\mathrm{i}\pi p/N})^{N}}{1 - e^{2\mathrm{i}\pi p/N}} = \frac{1 - e^{2\mathrm{i}\pi p}}{1 - e^{2\mathrm{i}\pi p/N}} = \frac{1 - 1}{1 - e^{2\mathrm{i}\pi p/N}} = 0.
\end{equation*}
We thus have that:
\begin{align*}
    \frac{1}{N}\sum_{k=0}^{N-1}\frac{\partial_{\boldsymbol{\theta}}F(\beta_{k})}{\beta_{k}} &= C_{1} + C_{0}\frac{1}{N|\beta|}\underbrace{\sum_{k=0}^{N-1}e^{-2\mathrm{i}\pi k/N}}_{=0} + \sum_{p=1}^{\infty}\frac{C_{p+1}}{(p+1)!}\frac{|\beta|^{p}}{N}\underbrace{\sum_{k=0}^{N-1}e^{2\mathrm{i}\pi pk/N}}_{=0 \text{~when~} p \not\equiv 0~(N)} \\
    \frac{1}{N}\sum_{k=0}^{N-1}\frac{\partial_{\boldsymbol{\theta}}F(\beta_{k})}{\beta_{k}} &= C_{1} +  \sum_{p \equiv 0 ~(N)}^{\infty}\frac{C_{p+1}|\beta|^{p}}{(p+1)!},
\end{align*}
which is the result of Eq.~\eqref{eq:quantitative}.

\subsection{Derivation of the online estimate}
\label{sec:app_proof_continuous}

Recall the formula of the online estimate (Eq.~\eqref{eq:continuous}):
\begin{equation*}
    \widetilde{\nabla}(T_{\text{plas}}) := -\frac{1}{T_{\text{plas}}|\beta|} \int_{0}^{T_{\text{plas}}} \sigma_{i}(t) \sigma_{j}(t) e^{-2 \mathrm{i} \pi t/T_{\text{osc}}}\mathrm{d}t.
\end{equation*}
If $T_{\text{dyn}} \ll T_{\text{osc}}$, the product of activities can be replaced by its value at the fixed point, and an exact gradient is computed after each period (Eq.~\eqref{eq:hopfield_case}).
Then if $T_{\text{osc}} \ll T_{\text{plas}}$, the integral can be divided into an integer amount of completed periods plus a remainder: $T_{\text{plas}} = k T_{\text{osc}} + T_{\text{rem}}$, where $k \in \mathbb{N}$ and $T_{\text{rem}} < T_{\text{osc}}$.
We then have by periodicity that:
\begin{equation*}
    \widetilde{\nabla}(T_{\text{plas}}) = \underbrace{\frac{k T_{\text{osc}}}{k T_{\text{osc}} + T_{\text{rem}}}}_{\to 1~\text{when}~T_{\text{plas}} \to \infty} \frac{\mathrm{d}\mathcal{L}}{\mathrm{d}w_{ij}}  - \underbrace{\frac{1}{k T_{\text{osc}} + T_{\text{rem}}}}_{\to 0~\text{when}~T_{\text{plas}} \to \infty} \frac{1}{|\beta|} \int_{0}^{T_{\text{rem}}} \sigma_{i}^{\ast}(t) \sigma_{j}^{\ast}(t) e^{-2 \mathrm{i} \pi t/T_{\text{osc}}}\mathrm{d}t.
\end{equation*}
In this way, when averaging over large $T_{\text{plas}}$, the number of completed cycles outweighs  the current period. Thus, by simply averaging over many oscillation cycles allows estimating gradients without explicit separate phases. 

\section{Detailed architecture}
\label{sec:app_detailed_archi}

\begin{figure}
  \centering
  \includegraphics[width=\textwidth]{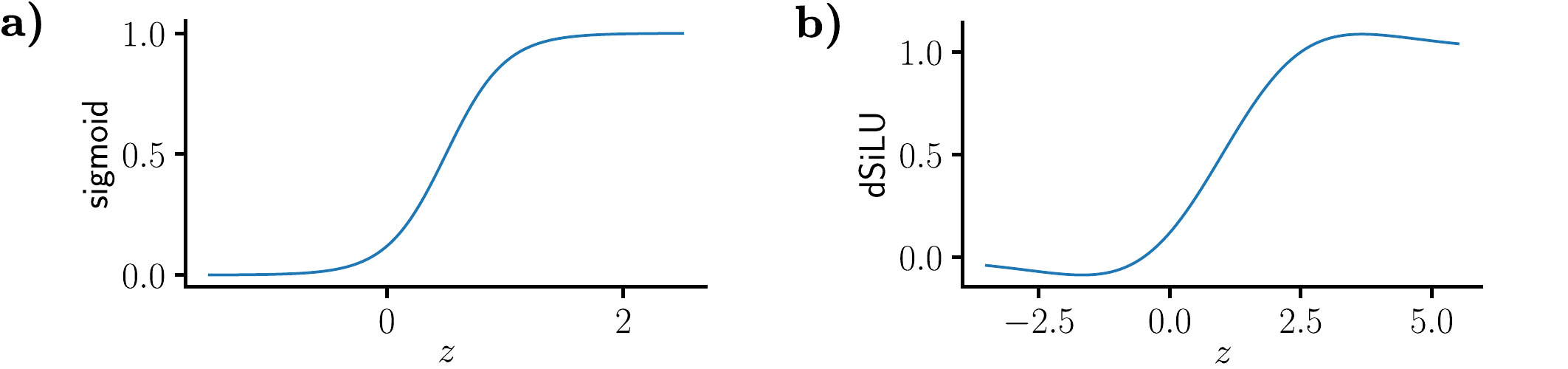}
  \caption{
  \textbf{a)} The shifted sigmoid we used in multi-layered perceptrons experiments.
  \textbf{b)} The dSiLU we used in \acp{CNN} experiments.
  }
  \label{fig:activation}
\end{figure}

\subsection{Dynamics for multi-layer perceptrons}
\label{sec:app_mlp_archi}

Assuming a number of $L$ layers, we note $\mathbf{s}_{l}$ the subset of units in layer $l$, with $\mathbf{s}_{0} = \mathbf{x}$ and $\mathbf{y}$ the one hot class label.
We note $\mathbf{W}_{l}$, and $\mathbf{b}_{l}$ the weight and biases of layer $l \geq 1$.
The energy function $F$ for a MLP optimizing the cross entropy loss reads:
\begin{equation*}
    F(\boldsymbol{\theta}, \mathbf{s}, \beta, \mathbf{y}) =  \sum_{l=1}^{L-1}\left[\frac{1}{2}\|\mathbf{s}_{l}\|^{2} - \sigma(\mathbf{s}_{l-1})^{\top}\cdot\mathbf{W}_{l}\cdot\sigma(\mathbf{s}_{l}) - \mathbf{b}_{l}^{\top} \cdot \sigma(\mathbf{s}_{l})\right] - \beta \mathbf{y}^{\top} \cdot \log(\mathbf{s}_{L}).
\end{equation*}
The activation function used for multi-layer perceptrons is the shifted sigmoid $z \mapsto 1/(1+e^{-4z+2})$ (Fig.~\ref{fig:activation}a).
We use the layer-wise discrete dynamics introduced by [\hyperlink{ernoult}{S3}, \hyperlink{laborieux}{S4}], which read:
\begin{align*}
\label{eq:mlp_dyn}
\left\{
\begin{array}{llll}
 \mathbf{s}_{l} &\leftarrow& \sigma \left( \mathbf{W}_{l} \mathbf{s}_{l-1} + {\mathbf{W}_{l+1}}^{\top} \mathbf{s}_{l+1} + \mathbf{b}_{l} + \boldsymbol{\eta}_{l} \right), & \text{for} ~ 1 \leq l \leq L-2 \\[.3cm]
 \mathbf{s}_{L-1} &\leftarrow& \sigma \left( \mathbf{W}_{L-1} \mathbf{s}_{L-2} + \beta {\mathbf{W}_{L}}^{\top} (\mathbf{y} - \mathbf{s}_{L}) + \mathbf{b}_{L-1} + \boldsymbol{\eta}_{L-1} \right), \\[.3cm]
 \mathbf{s}_{L} &\leftarrow& {\rm Softmax}\left(\mathbf{W}_{L} \mathbf{s}_{L-1} + \mathbf{b}_{L}\right),
\end{array}
\right.
\end{align*}
where $\boldsymbol{\eta}_{l}$ is an optional Gaussian noise added for Fig.~\ref{fig:cont_noise}c and Table \ref{tab:cont_noise}.
The noise was sampled at each time step.

\subsection{Dynamics for convolutional neural networks}
\label{sec:app_cnn_archi}

The activation function used for \acp{CNN} is a sigmoid-weighted linear unit [\hyperlink{dsilu}{S5}] (Fig.~\ref{fig:activation}b):
\begin{equation*}
    \operatorname{dSiLU}(z) := \left(\frac{z}{2}\right)\frac{1}{1+e^{-z}} + \left(1-\frac{z}{2}\right)\frac{1}{1+e^{-z+2}}.
\end{equation*}
We denote by $\mathcal{P}$ the pooling operation, and $\tilde{\mathcal{P}}$ the corresponding unpooling operation.
`$\ast$' denotes the convolution when preceded by $\mathbf{W}$ and transpose convolution when preceded by $\mathbf{W}^{\top}$.
The energy function $F$ for a CNN optimizing the cross entropy loss reads:
\begin{align*}
    &F(\boldsymbol{\theta}, \mathbf{s}, \beta, \mathbf{y}) =  \sum_{l \in \{\text{Conv}\}}\left[\frac{1}{2}\|\mathbf{s}_{l}\|^{2} - \sigma(\mathbf{s}_{l-1})^{\top}\cdot\mathcal{P}(\mathbf{W}_{l}\ast\sigma(\mathbf{s}_{l})) - \mathbf{b}_{l}^{\top} \cdot \sigma(\mathbf{s}_{l})\right] \\
    &\sum_{l \in \{\text{FC}\}}\left[\frac{1}{2}\|\mathbf{s}_{l}\|^{2} - \sigma(\mathbf{s}_{l-1})^{\top}\cdot\mathbf{W}_{l}\cdot\sigma(\mathbf{s}_{l}) - \mathbf{b}_{l}^{\top} \cdot \sigma(\mathbf{s}_{l})\right] - \beta \mathbf{y}^{\top} \cdot \log(\mathbf{s}_{L}).
\end{align*}
We use the layer-wise discrete dynamics introduced by [\hyperlink{ernoult}{S3}, \hyperlink{laborieux}{S4}], which read:
\begin{align*}
%\nonumber
\label{eq:cnn_dyn}
\left\{
\begin{array}{llll}
\displaystyle \mathbf{s}_{l} & \leftarrow & \sigma \left( \mathcal{P}(\mathbf{W}_{l} \ast \mathbf{s}_{l-1}) + \mathbf{W}_{l+1}^{\top} \ast \tilde{\mathcal{P}}(\mathbf{s}_{l+1}) + \mathbf{b}_{l} \right), & \text{for}~l \in \{\text{Conv layers}\} \\[.3cm]
 \mathbf{s}_{l} &\leftarrow& \sigma \left( \mathbf{W}_{l} \mathbf{s}_{l-1} + {\mathbf{W}_{l+1}}^{\top} \mathbf{s}_{l+1} + \mathbf{b}_{l} \right), & \text{for}~l \in \{\text{FC layers}\} \\[.3cm]
 \mathbf{s}_{L-1} &\leftarrow& \sigma \left( \mathbf{W}_{L-1} \mathbf{s}_{L-2} + \beta {\mathbf{W}_{L}}^{\top} (\mathbf{y} - \mathbf{s}_{L}) + \mathbf{b}_{L-1} \right), \\[.3cm]
 \mathbf{s}_{L} &\leftarrow& {\rm Softmax}(\mathbf{W}_{L} \mathbf{s}_{L-1} + \mathbf{b}_{L}).
\end{array}
\right.
\end{align*}
We used Softmax pooling [\hyperlink{softpool}{S6}] with a tunable temperature $\tau$, instead of the non-holomorphic Max pooling.
The output $y$ of Softmax pooling of an input $\mathbf{x}$ is defined for a kernel neighbourhood $\mathbf{R}$ by:
\begin{equation*}
    y = \sum_{i \in \mathbf{R}} \left( \frac{e^{x_{i}/\tau}}{\sum_{j \in \mathbf{R}} e^{x_{j}/\tau}} \right) x_{i}.
\end{equation*}
Note that Softmax pooling interpolates between Average pooling ($\tau \to \infty$) and Max pooling ($\tau \to 0$).

\section{Layer-wise comparison of the gradient in a deep network}
\label{sec:app_sweep_beta}

Here we show in Fig.~\ref{fig:full_sweep} the complete layer-wise cosine similarities between the estimates of holomorphic EP for various $N$ and the true gradient computed by automatic differentiation.

\begin{figure}
  \centering
  \includegraphics[width=0.9\textwidth]{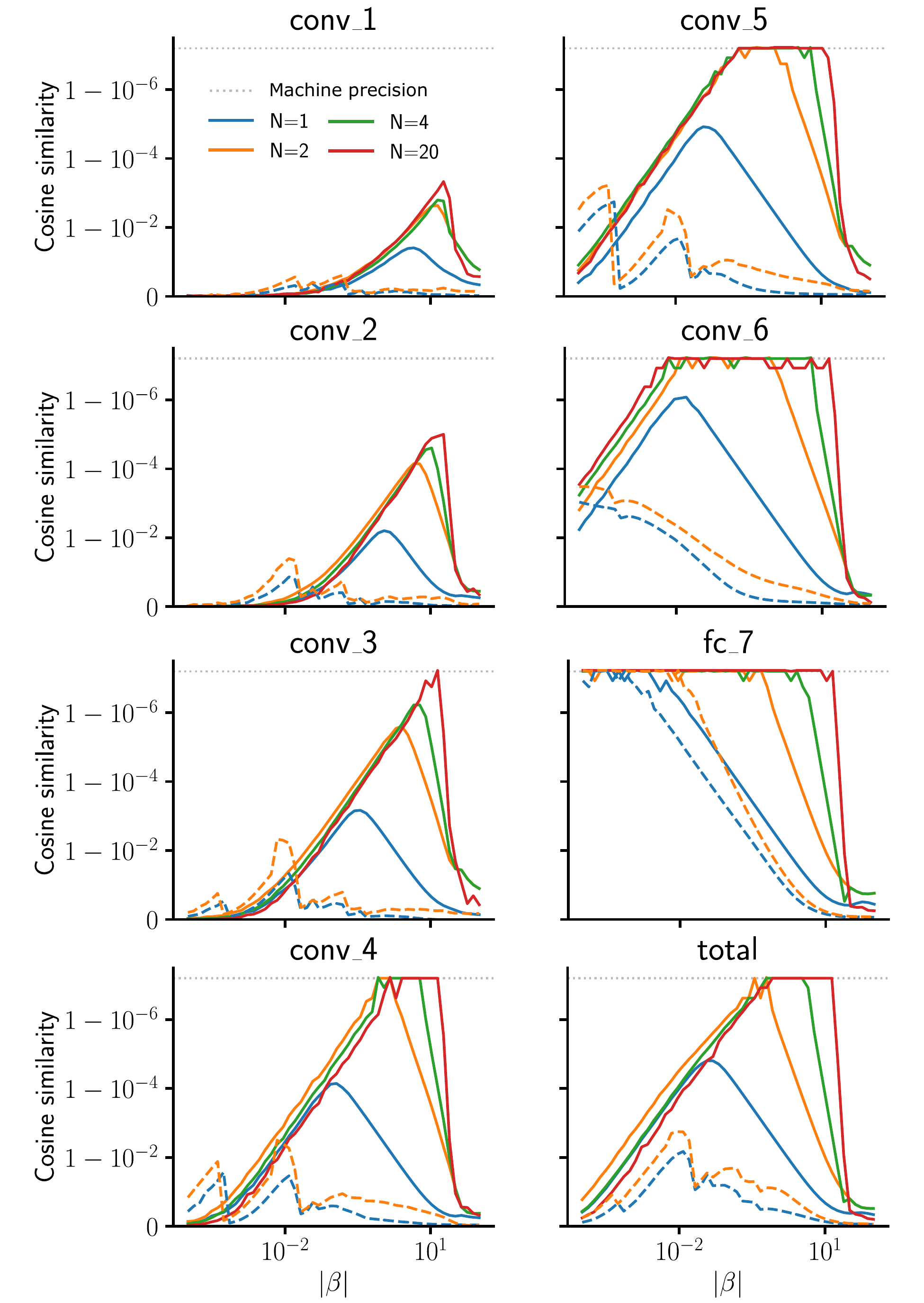}
  \caption{
  The complete layer-wise cosine similarity of Fig.~\ref{fig:sweep_beta}a).
  }
  \label{fig:full_sweep}
\end{figure}

\section{Dynamical stability in the complex plane}
\label{sec:app_fractal}

\begin{figure}
  \centering
  \includegraphics[width=0.9\textwidth]{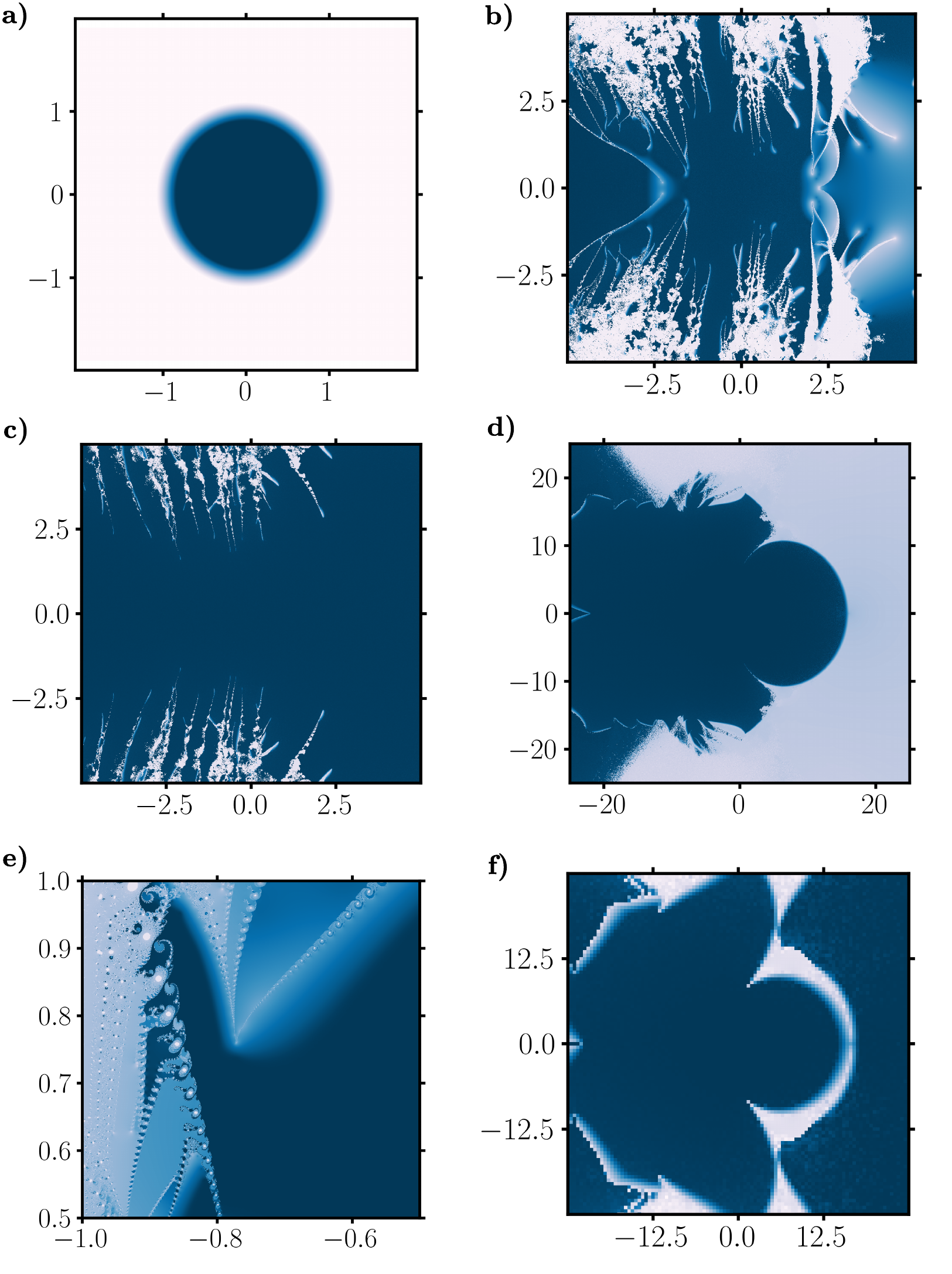}
  \caption{
  Map of convergence to a fixed point for complex $\beta$ in various settings.
  \textbf{a)}~\Ac{MLP} with linear activation function and low weight initialization.
  \textbf{b)}~\Ac{MLP} with shifted sigmoid activation and default weight initialization.
  \textbf{c)} Same as b) but with reduced weight initialization.
  \textbf{d)} Same as b) but with dSiLU activation function.
  \textbf{e)} Zoom at a frontier between stable and unstable regions.
  \textbf{f)} Small \ac{CNN} with dSiLU activation and Softmax pooling.
  }
  \label{fig:fractals}
\end{figure}

We show in Fig.~\ref{fig:fractals} how the area in $\mathbb{C}$ where the fixed point empirically exists varies with different architecture choices.
The data used for each panel is a digit from the MNIST dataset.
As in Fig.\ref{fig:overview}b), dark blue means that the fixed point exists, whereas light areas denote divergence.
These diverging areas could be due to the poles of the activation functions used.
For example, the sigmoid function $z \mapsto 1/(1+e^{-z})$ has $\{(2k+1)\mathrm{i}\pi ; k \in \mathbb{Z}\}$ as a set of poles where it diverges.
Although we did not systematically study this phenomenon in this work, we strongly suspect that these unstable areas are partly the result of the teaching signal being too strong or the weights being poorly distributed, thereby driving the complex neural activities near to the poles.
To some extent, the poles can be brought farther by introducing a coefficient in the exponential, but it results in flatter activation functions on the real axis, so a trade-off should be found.
In practice, we found that choosing reasonably the activation function, weight initialization, and teaching radius $|\beta|$ lead to enough stable areas around 0 to compute the gradient.

\section{Hyperparameters}
\label{sec:app_hyperparameters}

\subsection{MNIST experiments}
The digits were rescaled by 255 and flattened.
The hyperparameters used for training are reported in Table \ref{tab:mnist_hyper}, and the training errors are reported in Table~\ref{tab:app_cont_noise}.

\begin{table}[tbh]
  \caption{Hyperparameters used for the MNIST training experiment of Table \ref{tab:cont_noise}.}
  \label{tab:mnist_hyper}
  \centering
  \begin{tabular}{lccc}
    \toprule
    Hyperparameter & Classic EP &  \ac{hEP} & Online \ac{hEP} \\
    \midrule
    Batch size & 20 & 20 & 20 \\
    Learning rate & 5e-2 & 5e-2 & 5e-2 \\
    Epochs & 50 & 50 & 50 \\
    $|\beta|$ & 0.1 and 0.4 & 0.4 & 0.4 \\
    $T_{\text{free}}$ & 350 & 200 & 200* \\
    $T_{\text{nudge}}$ & 350 & 50 & N/A \\
    $T_{\text{osc}}$ & N/A & N/A & 300 \\
    $T_{\text{plas}}$ & N/A & N/A & 900 \\
    $N$ & N/A & 10 & 10 \\
    Noise** & 4e-2 & 4e-2 & 4e-2\\
    \bottomrule
    \multicolumn{4}{l}{* \footnotesize{Only used for evaluation}} \\
    \multicolumn{4}{l}{** \footnotesize{Standard deviation of the Gaussian noise for experiments with noise.}} \\
  \end{tabular}
\end{table}

\begin{table}[tbh]
  \caption{MNIST training and validation errors for classic EP \cite{scellier2017equilibrium}, \ac{hEP}, and online \ac{hEP}, with and without noise. 
  All results are averages ($n=3$) $\pm$ one standard deviation.}
  \label{tab:app_cont_noise}
  \centering
  \scalebox{0.85}{
  \begin{tabular}{lllllllll}
    \toprule
    & \multicolumn{2}{c}{Class. EP, $|\beta|=0.1$} & \multicolumn{2}{c}{Class. EP, $|\beta|=0.4$} &  \multicolumn{2}{c}{\ac{hEP}, $|\beta|=0.4$} & \multicolumn{2}{c}{Online \ac{hEP}} \\
    \cmidrule(r){2-9}
      Noise  & Train (\%) &  Val (\%) & Train (\%) &  Val (\%)  & Train (\%)  & Val (\%)  & Train (\%) & Val. (\%)\\
    \midrule
    No   & 0.05 $^{\pm 0.02}$ & 1.87 $^{\pm 0.01}$ & 0.19 $^{\pm 0.05}$ & 2.24 $^{\pm 0.05}$ & 0.02 $^{\pm 0.01}$  & 1.97 $^{\pm 0.08}$ & 0.11 $^{\pm 0.01}$ & 2.05 $^{\pm 0.02}$ \\
    Yes  & 88.8 $^{\pm 0.0}$  & 88.7 $^{\pm 0.0}$ & 1.96 $^{\pm 0.2}$  & 3.01 $^{\pm 0.1}$ &  0.14 $^{\pm 0.03}$ & 1.96 $^{\pm 0.07}$ & 0.13 $^{\pm 0.03}$ & 1.91 $^{\pm 0.16}$\\
    \bottomrule
  \end{tabular}
  }
\end{table}

\subsection{Large-scale experiments}

In the training experiments for CIFAR-10, CIFAR-100, and ImageNet $32 \times 32$, the training data was normalized, then augmented with 50\% chance random horizontal flips, resized to $36 \times 36$ resolution with padding, and cropped randomly back to $32 \times 32$.
The optimizer used was stochastic gradient descent with momentum.
Pooling was applied at all layers for the five-layer CNN, and every other layer starting with the first layer in the seven-layer CNN.

\begin{table}[tbh]
  \caption{Hyperparameters used for the VGG training experiments of Table \ref{tab:perf} and Fig.\ref{fig:sweep_beta}c.}
  \label{tab:large_scale_hyper}
  \centering
   \scalebox{0.97}{
  \begin{tabular}{lcccc}
    \toprule
    Hyperparameter & CIFAR-10 &  CIFAR-100 & ImageNet $32 \times 32$ & CIFAR-10 (Fig.\ref{fig:sweep_beta}c) \\
    \midrule
    Batch size & 128 & 128 & 256 & 128 \\[0.2cm]
    Channel sizes & \multicolumn{3}{c}{[128, 256, 512, 512]} & \small{[128, 128, 256, 256, 512, 512]} \\
    Kernel sizes & \multicolumn{3}{c}{[3, 3, 3, 3]} & \small{[3, 3, 3, 3, 3, 3]}\\
    Strides & \multicolumn{3}{c}{[1, 1, 1, 1]} & \small{[1, 1, 1, 1, 1, 1]}\\
    Paddings & \multicolumn{3}{c}{[1, 1, 1, 0]} & \small{[1, 1, 1, 0, 1, 0]}\\
    SoftPool window & \multicolumn{3}{c}{$2 \times 2$} & $2 \times 2$\\
    SoftPool stride & \multicolumn{3}{c}{2} & 2\\
    SoftPool temp. & \multicolumn{3}{c}{1} & 10\\[0.2cm]
    Initial LRs* & \multicolumn{3}{c}{\small{[25, 15, 10, 8, 5] $\times$ 1e-2}} & \small{[5, 4, 4, 3, 3, 2, 2] $\times$ 1e-2}\\
    Final LRs & \multicolumn{3}{c}{\small{[25, 15, 10, 8, 5] $\times$ 1e-9}} & \small{[5, 4, 4, 3, 3, 2, 2] $\times$ 1e-9} \\[0.2cm]
    Weight decay & 2e-3 & 1e-2 & 5e-4 & \small{[5, 5, 5, 5, 5, 5, 10] $\times$ 1e-4}\\
    Momentum & 0.9 & 0.9 & 0.9 & 0.9\\
    Epochs & 90 & 90 & 90 & 90\\
    $|\beta|$ & 1.0 & 1.0 & 1.0 & 1.0\\
    $T_{\text{free}}$ & 250 & 250 & 250 & 260\\
    $T_{\text{nudge}}$ & 60 & 60 & 60 & 60\\
    $N$ & 2 & 2 & 2 & 2 and 4\\
    \bottomrule
    \multicolumn{5}{l}{* \footnotesize{Learning rates were decayed with cosine annealing without restart [\hyperlink{cosine_annealing}{S7}].}} \\
  \end{tabular}
  }
\end{table}

\section{Simulations details}
\label{sec:app_simulations}

All simulations were performed on an in-house GPU cluster or workstations.
Each simulation in Table~\ref{tab:perf} was run in parallel on four NVIDIA V100 GPUs.
The training runs on ImageNet $32 \times 32$ took 5.5 days each for EP, and a few hours for BP.
The runs on CIFAR-10 and CIFAR-100 took one day on average depending on the architecture (5 or 7 layers) and the number of time steps used for the dynamics.
The use of complex numbers, although seamlessly implementable with Jax, results in longer simulation times due to the 64 bit-precision requirement (32 bit-precision for real and imaginary parts respectively).

\newpage
\section*{Supplementary References}

{

[\hypertarget{appel}{S1}] Walter Appel.
\newblock Mathematics for physics and physicists.
\newblock 2007.

[\hypertarget{cartan}{S2}] Henri Cartan.
\newblock \emph{Th{\'e}orie {\'e}l{\'e}mentaire des fonctions analytiques d'une
  ou plusieurs varibales complexes: Avec le concours de Reiji Takahashi}.
\newblock Hermann, 1961.

[\hypertarget{ernoult}{S3}] Maxence Ernoult, Julie Grollier, Damien Querlioz, Yoshua Bengio, and Benjamin
  Scellier.
\newblock Updates of equilibrium prop match gradients of backprop through time
  in an rnn with static input.
\newblock \emph{Advances in neural information processing systems}, 32, 2019.

[\hypertarget{laborieux}{S4}] Axel Laborieux, Maxence Ernoult, Benjamin Scellier, Yoshua Bengio, Julie
  Grollier, and Damien Querlioz.
\newblock Scaling equilibrium propagation to deep convnets by drastically
  reducing its gradient estimator bias.
\newblock \emph{Frontiers in neuroscience}, 15:\penalty0 129, 2021.

[\hypertarget{dsilu}{S5}] Stefan Elfwing, Eiji Uchibe, and Kenji Doya.
\newblock Sigmoid-weighted linear units for neural network function
  approximation in reinforcement learning.
\newblock \emph{Neural Networks}, 107:\penalty0 3--11, 2018.

[\hypertarget{softpool}{S6}] A~Stergiou, R~Poppe, and G~Kalliatakis.
\newblock Refining activation downsampling with softpool. arxiv 2021.
\newblock \emph{arXiv preprint arXiv:2101.00440}.

[\hypertarget{cosine_annealing}{S7}] Ilya Loshchilov and Frank Hutter.
\newblock Decoupled weight decay regularization.
\newblock \emph{arXiv preprint arXiv:1711.05101}, 2017.

}

\end{document}